\documentclass{article}

\PassOptionsToPackage{numbers, compress}{natbib}

\usepackage[main, final]{neurips_2026}

\usepackage[utf8]{inputenc}
\usepackage[T1]{fontenc}
\usepackage{hyperref}
\usepackage{url}
\usepackage{booktabs}
\usepackage{amsfonts}
\usepackage{nicefrac}
\usepackage{xcolor}

\usepackage{amsmath}
\usepackage{amssymb}
\usepackage{mathtools}
\usepackage{amsthm}
\usepackage{pifont}
\usepackage{multirow}
\usepackage{makecell}
\usepackage{algorithm}
\usepackage{algorithmic}
\usepackage{float}
\usepackage{wrapfig}
\usepackage{caption}
\usepackage{fvextra}
\usepackage{float}
\usepackage[most]{tcolorbox}
\tcbuselibrary{breakable}
\DefineVerbatimEnvironment{SampleVerbatim}{Verbatim}{breaklines,breakanywhere,breaksymbolleft={},breaksymbolright={},fontsize=\small}
\newtcolorbox{samplebox}[1]{
  width=\textwidth,
  colback=white,
  colframe=black!25,
  coltitle=black,
  colbacktitle=black!8,
  fonttitle=\bfseries,
  title=#1,
  boxrule=0.4pt,
  arc=2pt,
  left=8pt,
  right=8pt,
  top=6pt,
  bottom=6pt
}

\usepackage{tocloft} 
\usepackage{bm,bbm}

\newcommand{\x}{\mathbf{x}}

\newcommand{\z}{\mathbf{z}}
\newcommand{\y}{\mathbf{y}}

\newcommand{\cat}{\mathrm{Cat}}

\newcommand{\onehots}{\mathcal{V}}
\newcommand{\vocabsize}{V}

\newcommand{\at}{\alpha_{t}}
\newcommand{\as}{\alpha_{s}}
\newcommand{\atdot}{\dot{\at}}

\usepackage[dvipsnames,table]{xcolor}
\definecolor{ourblue}{rgb}{0.368,0.507,0.71}
\definecolor{ourorange}{rgb}{0.881,0.611,0.142}
\definecolor{ourgreen}{RGB}{90, 170, 100}
\definecolor{ourred}{rgb}{0.923,0.386,0.209}
\definecolor{ourviolet}{rgb}{0.528,0.471,0.701}
\definecolor{ourbrown}{rgb}{0.772,0.432,0.102}
\definecolor{ourlightblue}{rgb}{0.364,0.619,0.782}
\definecolor{ourdarkgreen}{rgb}{0.572,0.586,0.}
\definecolor{url}{HTML}{d95225}
\definecolor{bloodred}{HTML}{B00000}
\definecolor{ourgray}{HTML}{36454F}

\definecolor{ourcyan2}{rgb}{0.125,0.722,0.804}
\definecolor{ourred2}{rgb}{0.863,0.184,0.047}
\definecolor{ouryellow2}{cmyk}{0,0.16,1.0,0.07}
\definecolor{ourviolet2}{cmyk}{0.55,0.56,0,0.47}
\definecolor{ourorange2}{cmyk}{0,0.46,0.89,0.11}

\definecolor{figblue}{HTML}{4169E1}
\definecolor{figgreen}{HTML}{33CD32}
\definecolor{figorange}{HTML}{FFA503}
\definecolor{figred}{HTML}{FF0000}
\definecolor{figgray}{HTML}{808080}

\def\eqref#1{equation~\ref{#1}}

\def\1{\bm{1}}

\DeclareMathAlphabet{\mathsfit}{\encodingdefault}{\sfdefault}{m}{sl}
\SetMathAlphabet{\mathsfit}{bold}{\encodingdefault}{\sfdefault}{bx}{n}

\newcommand{\E}{\mathbb{E}}

\newcommand{\KL}{D_{\mathrm{KL}}}

\newcommand{\mask}{\boldsymbol{m}}

\DeclareMathOperator*{\argmax}{arg\,max}

\DeclareFontFamily{OT1}{pzc}{}
\DeclareFontShape{OT1}{pzc}{m}{it}{<-> s * [1.10] pzcmi7t}{}
\DeclareMathAlphabet{\mathpzc}{OT1}{pzc}{m}{it}

\newcommand{\method}{$\mathpzc{ProSeCo}$}
\newcommand{\masktoken}{\texttt{[M]}}

\newcommand{\xL}{\x^{1:L}}
\newcommand{\xl}{\x^{\ell}}

\newcommand{\zt}{\z_t}
\newcommand{\ztL}{\zt^{1:L}}
\newcommand{\ztl}{\zt^{\ell}}

\newcommand{\zs}{\z_s}

\newcommand{\yL}{\y^{1:L}}
\newcommand{\yl}{\y^{\ell}}
\newcommand{\yt}{\y_t}
\newcommand{\ytL}{\yt^{1:L}}
\newcommand{\ytl}{\yt^{\ell}}

\newcommand{\sg}{\mathrm{sg}}

\newcommand{\hzt}{\hat{\z}_t}
\newcommand{\hztL}{\hzt^{1:L}}

\newcommand{\hx}{\hat{\x}}
\newcommand{\hxL}{\hx^{1:L}}

\hypersetup{
    final,
    colorlinks,
    linkcolor=ourred,
    citecolor=ourblue,
    urlcolor=url,
}

\newcommand{\xmark}{\ding{55}}

\usepackage[capitalize,noabbrev]{cleveref}

\theoremstyle{plain}
\newtheorem{theorem}{Theorem}[section]
\newtheorem{proposition}[theorem]{Proposition}

\theoremstyle{definition}

\theoremstyle{remark}

\title{Learn from Your Mistakes:\\Self-Correcting Masked Diffusion Models}

\author{
    \textbf{Yair Schiff}$^{~1,*}$, \quad
    \textbf{Omer Belhasin}$^{2,*}$, \quad
    \textbf{Roy Uziel}$^{2}$, \quad
    \textbf{Guanghan Wang}$^{1}$, \quad
    \textbf{Marianne Arriola}$^{1}$ \\[0.3em]
    \textbf{Gilad Turok}$^{1}$, \quad
    \textbf{Ran Zilberstein}$^{2}$, \quad
    \textbf{Michael Elad}$^{2,\dagger}$, \quad
    \textbf{Volodymyr Kuleshov}$^{1,\dagger}$ \\[0.5em]
    \normalfont
    $^{1}$ Cornell University
    \qquad
    $^{2}$ NVIDIA \\[0.3em]
    \normalfont\footnotesize
    $^{*}$ Equal contribution
    \qquad
    $^{\dagger}$ Equal senior authorship
}

\begin{document}

\maketitle

\begingroup
\renewcommand{\thefootnote}{}
\footnotetext{
    Correspondence to:
    Yair Schiff
    \href{mailto:yairschiff@cs.cornell.edu}{\texttt{yairschiff@cs.cornell.edu}}
    and Omer Belhasin
    \href{mailto:obelhasin@nvidia.com}{\texttt{obelhasin@nvidia.com}}.
}
\endgroup

\begin{abstract}
Masked diffusion models (MDMs) have emerged as a promising alternative to autoregressive models, enabling parallel token generation while achieving competitive performance. Despite these advantages, MDMs face a fundamental limitation: once tokens are unmasked, they remain fixed, leading to error accumulation and ultimately degrading sample quality. We address this by proposing a framework that trains a model to perform both unmasking and correction. By reusing outputs from the MDM denoising network as inputs for corrector training, we train a model to recover from potential mistakes. During generation we apply additional corrective refinement steps between unmasking ones in order to change decoded tokens and improve outputs. We name our training and sampling method $\mathpzc{Pro}$gressive $\mathpzc{Se}$lf-$\mathpzc{Co}$rrection  (\method) for its unique ability to iteratively refine an entire sequence, including already generated tokens.
We conduct extensive experimental validation across multiple conditional and unconditional tasks, demonstrating that \method~yields better quality-efficiency trade-offs (up to $\sim$4$\times$ faster sampling while still outperforming the best accuracy baseline method) and enables further gains through inference-time compute scaling (up to $\sim$1.2$\times$ improvement on benchmarks when allowing for slower, high-NFE decoding).
We provide the code\footnote{Code: \url{https://github.com/kuleshov-group/proseco}}, model weights, and blog post on the project page:
\par
{\centering
\url{https://yair-schiff.github.io/proseco}.\par
}
\end{abstract}

\begin{figure}[ht!]
\centering
\vspace{-1em}
\hspace{-6em}
\begin{minipage}[c]{0.7\linewidth}
    \includegraphics[width=0.85\linewidth]{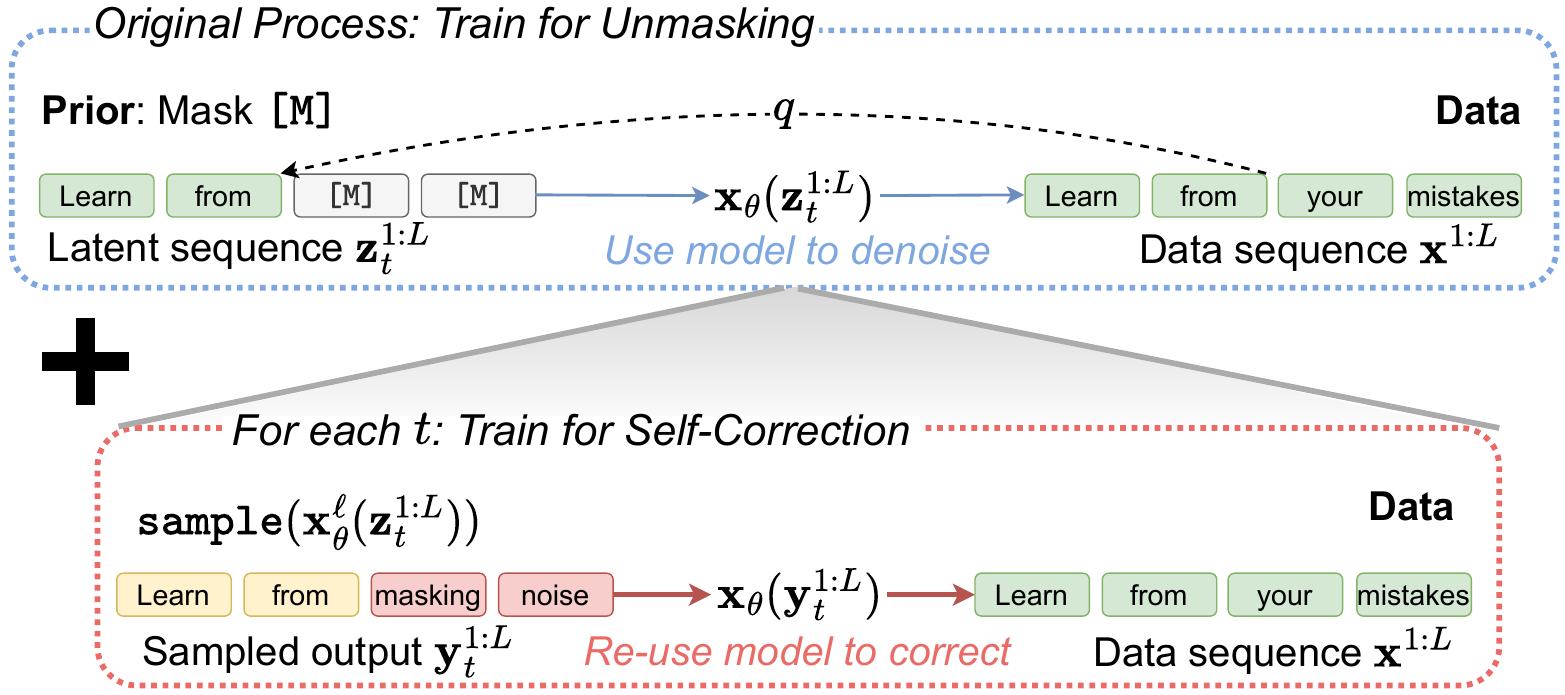}
    \label{fig:graphical_abtract_schematic}
\end{minipage}
\hspace{-4.5em}
\begin{minipage}[c]{0.3\linewidth}
    \includegraphics[width=1.5\linewidth]{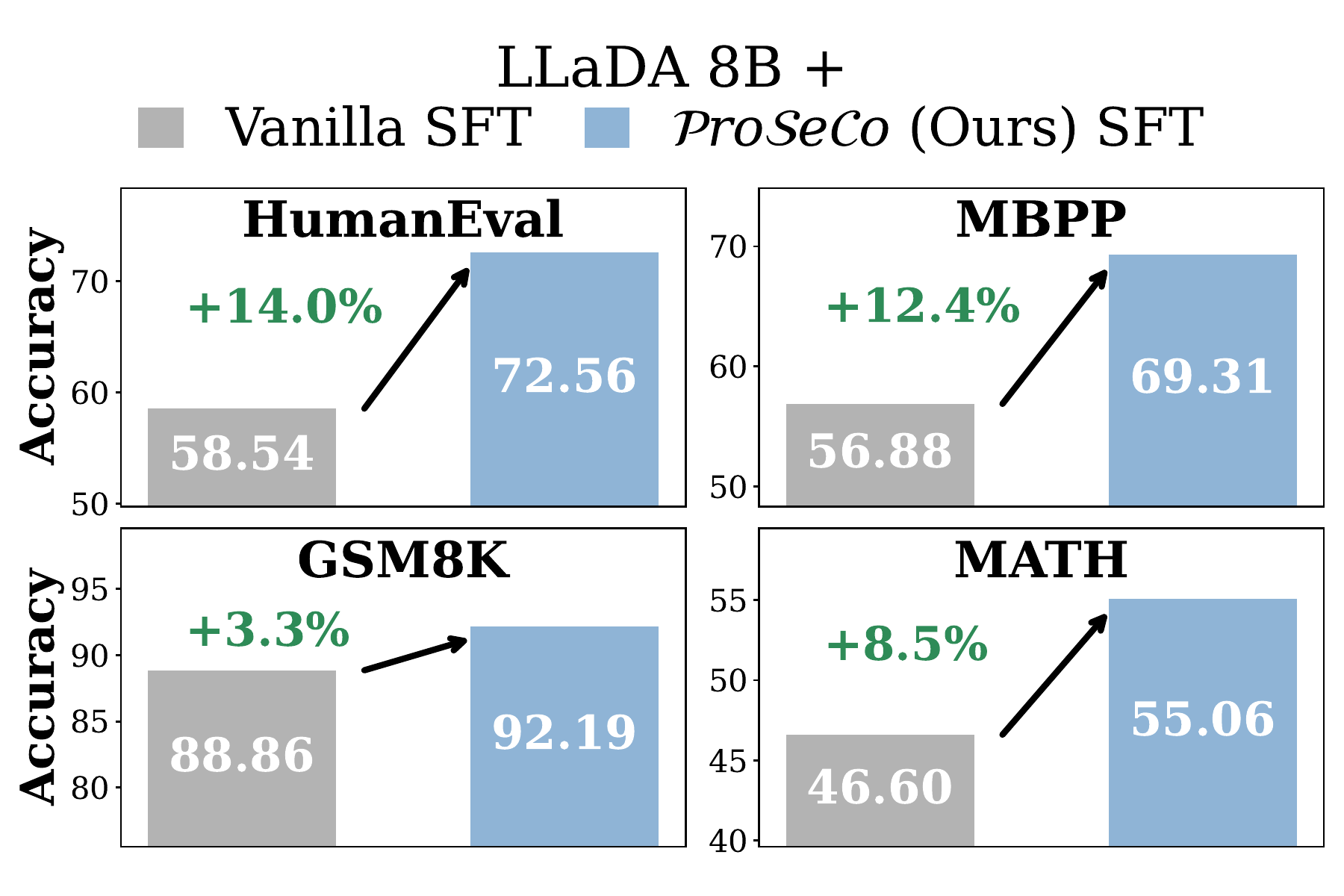}
\label{fig:graphical_abstract_results}
\end{minipage}
\vspace{-1em}
\caption{(\textit{Left}) Training \method: The original process trains for generation via unmasking.
For every timestep in the masking process, we also train for self-correction: undoing corruptions that arise from sampling from unmasking predictions.
\textit{(Right)} Using our method for SFT and Max correction sampling with LLaDA \citep{nie2025large} significantly outperforms vanilla masked diffusion SFT.
}
\end{figure}
\vspace{-1em}

\section{Introduction}\label{sec:intro}
Masked diffusion models (MDMs) \citep{loudiscrete, ou2024your, sahoo2024simple, shi2024simplified} have emerged as a powerful paradigm for discrete data generation and offer a compelling alternative to autoregressive (AR) models.
Treating generation as a denoising process that gradually unmasks tokens in parallel, MDMs can achieve efficiency gains and maintain quality across various domains, and have even demonstrated competitive performance from the 8B scale \citep{nie2025large} up to 100B parameters \citep{bie2025llada2}.

However, a fundamental limitation persists for MDMs: once a token is unmasked, it remains fixed for the duration of the generation process.
Consequently, errors made during parallel decoding inevitably accumulate, leading to distributional drift and degraded sample quality.
While recent work has begun to explore error correction for MDMs \citep{huang2025don, kim2025fine, lezama2022discrete, liu2026teach, wang2025remasking, zhao2024informed}, efficiently identifying which tokens require modification and altering them remains  a significant challenge.

In this work, we address this limitation by proposing a principled framework that equips MDMs with the inherent ability to both decode and correct.
Our key insight is to treat model-generated outputs as corrupted versions of the true data, where errors represent a form of noise that can be undone through a nested corrector loop.
Our method therefore trains a model to recover the clean signal from its own potentially mistaken outputs, enabling it to learn from and correct its own failure modes.
We implement this training via a simple additional cross-entropy loss term added to standard MDM objectives.
During inference, we leverage this capability by interleaving corrective steps in between standard unmasking steps, allowing the model to dynamically refine and ``self-correct" outputs.
The training and inference modifications we make to existing MDM algorithms are minimal, yet lead to marked gains. 
We name our approach $\mathpzc{Pro}$gressive $\mathpzc{Se}$lf-$\mathpzc{Co}$rrection (\method),
since our method has the ability to iteratively refine all positions in a sequence, including ones already unmasked.

On math and coding benchmarks, our method significantly improves standard fine-tuning of large MDMs, enabling up to $\sim$4$\times$ faster generation without quality degradation, and up to $\sim$1.2$\times$ increase in benchmark accuracy.
For guided generation, we demonstrate that recovering from mistaken tokens improves the Pareto frontier of sample quality and property maximization.
Finally, we also demonstrate that for unconditional generation, our method improves over MDMs and other proposed correctors in terms of generating fluent text without collapsing the diversity of generated outputs.
Across all experiments we demonstrate that \method~better trades-off quality and speed, and can also benefit from inference-time scaling to further increase sample quality.

In summary, our contributions are as follows:
    1) We present a framework that jointly trains for decoding masked tokens and correcting mistakes.
    2) We provide easy-to-implement training and sampling algorithms that entail only minor additions to standard MDMs.
    3) We conduct comprehensive experiments demonstrating that \method~outperforms quality-efficiency trade-offs of baseline methods and enables inference-time compute scaling to further improve quality beyond MDMs.

\section{Background}\label{sec:background}
\textbf{Discrete Diffusion}~
Diffusion is a paradigm for generative modeling where a denoising process $p_\theta$ is trained to undo a pre-defined corruption process $q$ \citep{sohl2015deep, song2019generative, song2020score}.
Starting from data $\x \sim q_{\text{data}}$, the corruption process produces latent variables $\zt \sim q(\zt \mid \x)$ for $t \in [0, 1]$, which increasingly move further from the data and towards noise, as $t$ increases.

Adapting diffusion models to discrete data requires a corruption process over the space of sequences of tokens with values in a finite vocabulary \citep{austin2021structured}. 
To denote this data, we let $\mathbf{x}, \mathbf{z}_t \in \onehots$, where $\onehots := \{0, 1\}^\vocabsize \subset \Delta^\vocabsize$, and $\Delta^\vocabsize$ denotes the probability simplex over $\vocabsize$ categories. 
We use superscripts to denote the sequence dimension; for example, $\mathbf{x}^{1:L} \in \onehots^L$ represents a sequence of tokens $(\mathbf{x}^1, \dots, \mathbf{x}^L)$, where token $\mathbf{x}^\ell \in \onehots$, for $\ell \in \{1, \dots, L\}$.

\textbf{Masked Diffusion Models}~
A promising instantiation of the discrete diffusion paradigm is the recent line of masked diffusion models (MDMs; \cite{loudiscrete, ou2024your, sahoo2024simple, shi2024simplified}).
In this framework, the corruption process is characterized by marginals that interpolate between data and noise:
$q(\zt \mid \x) = \cat(\zt; \at \x + (1-\at) \mask),$
where $\mask$ denotes the one-hot vector for a special \masktoken~token and $\at := \alpha(t) \in [0, 1]$ is a noise schedule that is monotonically decreasing in $t$.
In MDM, the corruption process is defined to be `absorbing,' meaning that once a token transitions to the masked state, it remains in this state.

Diffusion models are trained via a variational bound on the negative log-likelihood (NLL).
This bound encourages the learned approximate posterior $p_\theta(\z_s \mid \zt)$ to match the true one $q(\z_s \mid \x, \zt)$, for $s < t$ \citep{sohl2015deep}.
For MDMs, the true posteriors take the following form \citep{austin2021structured, sahoo2024simple}:
\begin{align}\label{eq:mdlm_true_posterior}
    q(\z_s \mid \x, \zt \neq \mask) =
        \cat(\z_s; \zt), \quad
    q(\z_s \mid \x, \zt = \mask) =
        \cat\left(\z_s; \tfrac{\as - \at}{1-\at}\x + \tfrac{1-\as}{1-\at}\mask\right).
\end{align}
A common parameterization for $p_\theta(\z_s \mid \zt)$ replaces $\x$ in (\ref{eq:mdlm_true_posterior}) with the output of a neural network: $\x_\theta(\zt) \in \Delta^\vocabsize$.
Taking the continuous-time limit of the diffusion process, the variational objective for MDMs simplifies to \citep{ou2024your, sahoo2024diffusion, shi2024simplified}:
\begin{align}\label{eq:mdlm_nelbo}
    \mathcal{L}^{\mathpzc{MDM}}_{\theta} :=
    \E_{q_{\text{data}}}\int_0^1\E_{q_t} \sum_{\ell=1}^L \delta_{\zt^\ell, \mask} \frac{\atdot}{1-\at}\log \langle \xl_\theta(\zt^{1:L}), \xl \rangle dt,
\end{align}
where $\atdot$ denotes the time derivative of the noise schedule, $\delta_{\mathbf{a}, \mathbf{b}} := \delta(\mathbf{a}, \mathbf{b})$ is the Kronecker delta function that evaluates to $1$ if $\mathbf{a} = \mathbf{b}$, $\E_{q_{\text{data}}}$ refers to the expectation over drawing data samples $\x^{1:L}\sim q_{\text{data}}$, and $\E_{q_t}$ is the expectation over drawing noisy latents from the forward marginals $\zt^{1:L}\sim q(\zt^{1:L}\mid \x^{1:L})$.
We optimize this objective via stochastic gradient descent using Monte Carlo samples for the expectations and $t \sim \mathcal{U}[0,1]$ to approximate the integral.

A key drawback of this paradigm is that the denoising network $\x_\theta$ does not learn to modify unmasked inputs $\zt \neq \mask$.
Without the ability to correct mistakes, during generation, errors propagate and accumulate over time, causing the sampling trajectory to deviate from the true data distribution.

\begin{figure}[t]
    \centering
    \hspace{-2em}
    \begin{minipage}[c]{0.32\linewidth}
    \includegraphics[width=1.1\linewidth]{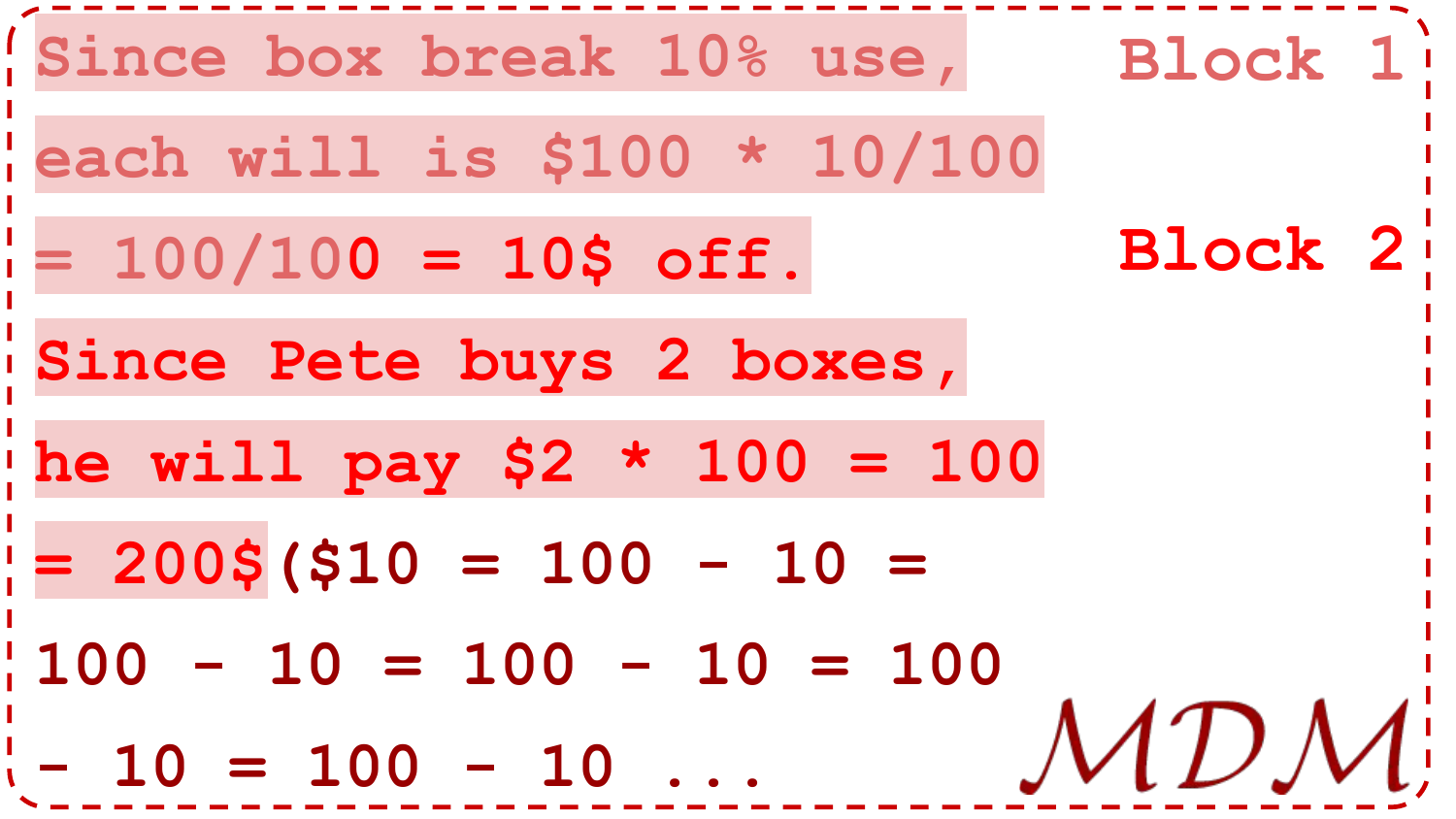}
    \end{minipage}
    \hspace{1em}
    \begin{minipage}[c]{0.32\linewidth}
    \includegraphics[width=1.1\linewidth]{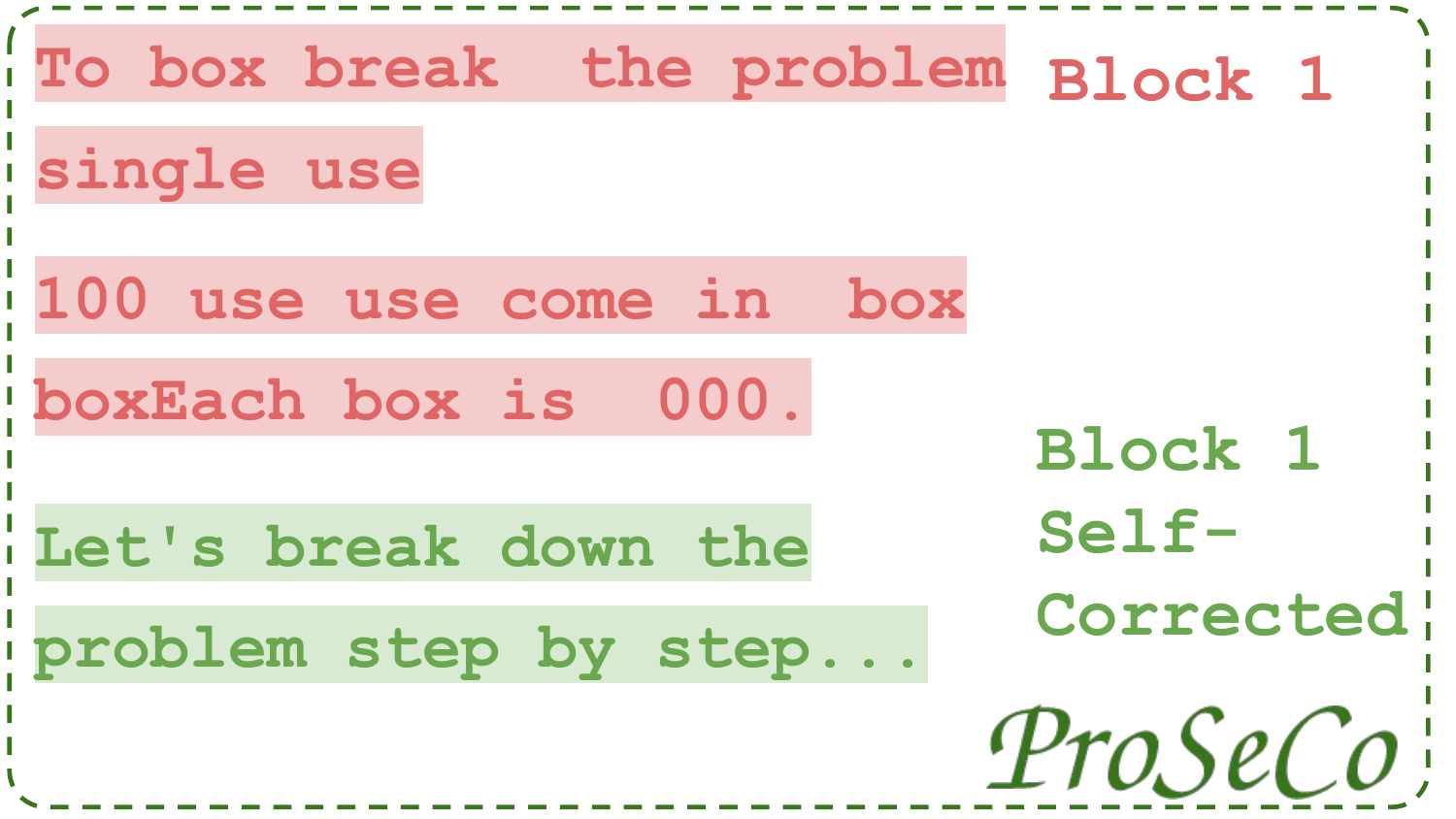}
    \end{minipage}
    \hspace{1em}
    \begin{minipage}[c]{0.32\linewidth}
    \includegraphics[width=1.1\linewidth]{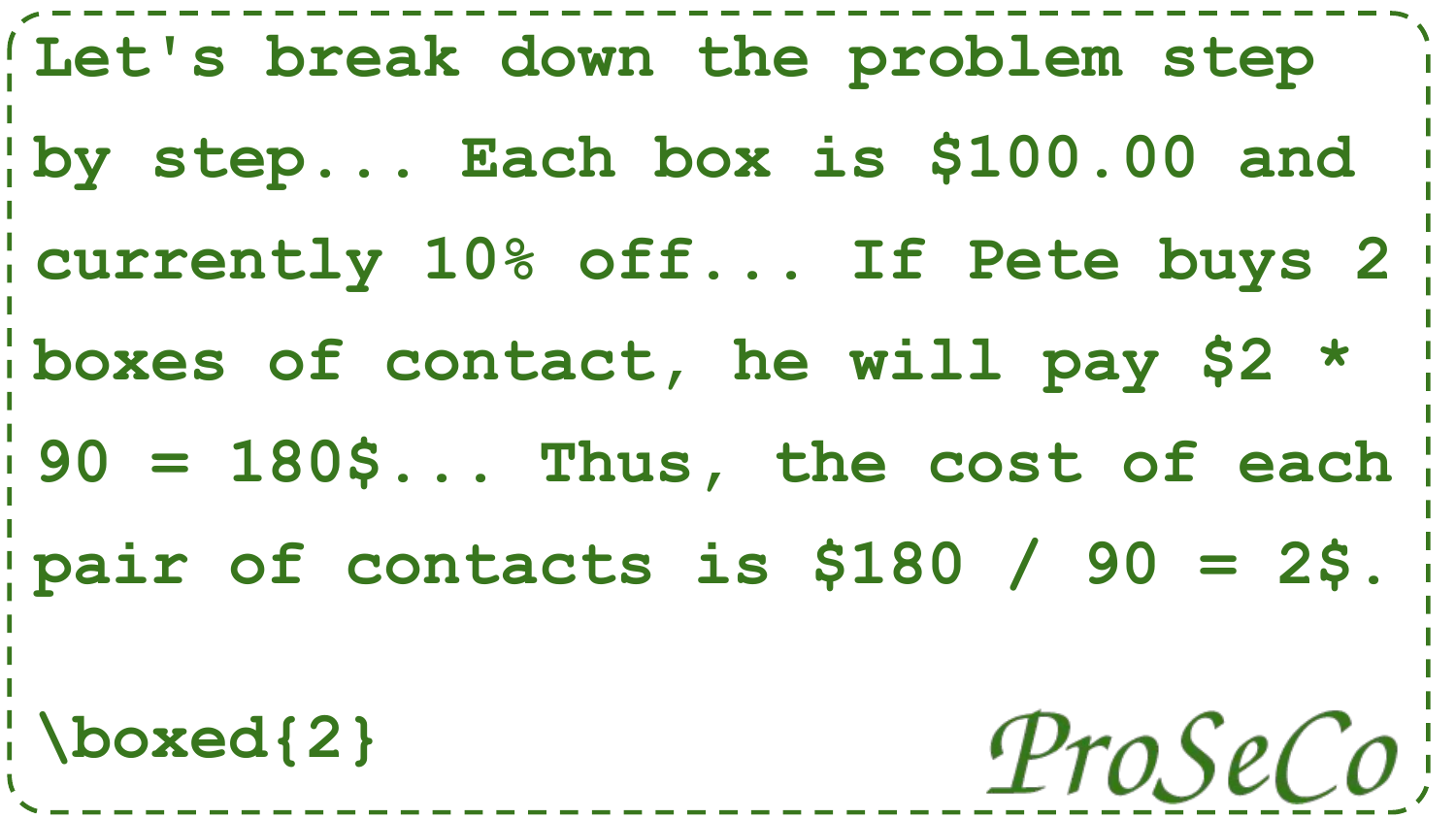}
    \end{minipage}
    \caption{Illustrative example demonstrating the benefits of self-correction.
    \textit{(Left)} For standard MDM, during parallel unmasking, errors occur.
    Mistakes accumulate and lead to sample collapse.
    \textit{(Middle)} For \method, a short correction loop steers generation `back on track.'
    \textit{(Right)} \method~can thus produce a final high quality output even with parallel generation.
    }
    \label{fig:self_correction}
\end{figure}
\section{Self-Correcting Masked Diffusion Models}\label{sec:methods}
In this work, we aim to equip MDMs with the ability to alter previously decoded tokens.
Our approach is to have a single model that can act in two `modes': when inputs contain masked tokens, the model's role is to unmask; when inputs contain all non-mask tokens, the model operates in `corrector' mode and can update already generated positions.

Below, we present an augmented objective that trains models to jointly unmask and correct mistaken unmasked tokens.
Our key insight is that during generation, $\x_\theta$ produces unmasked tokens at each position, some of which align with the true data distribution and some of which potentially contain errors or distributional misalignment.
A subset of these tokens are then carried into the next round of iterative refinement.
We can interpret outputs of $\x_\theta$ as corrupted sequences from the data distribution, where certain tokens have been replaced by incorrect ones sampled from $\x_\theta$.
This perspective motivates our formulation of model-generated sequences serving as inputs for a corrector model.

In Section \ref{subsec:methods_corrector_loss}, we introduce a straightforward corrector loss that we add to the MDM objective from (\ref{eq:mdlm_nelbo}), and we detail a training algorithm to efficiently and jointly train a model to unmask and correct decoding errors.
Finally, in Section \ref{sec:methods_sampling}, we provide a sampling algorithm that leverages these abilities to further improve sample quality by interleaving error correction predictions and unmasking steps.
Owing to this enhanced ability to iteratively refine any part of the sequence, including already decoded tokens, we dub our method $\mathpzc{Pro}$gressive $\mathpzc{S}$elf-$\mathpzc{Co}$rrection (\method).

\subsection{Self-Correcting Objective}\label{subsec:methods_corrector_loss}
The high level idea of our approach is to train a corrector to minimize error between outputs from the denoiser and clean data.
By tying weights of the corrector and denoiser networks, we turn this formulation into an augmented MDM variational objective with an auxiliary error correction loss.

More precisely, let $\pi: \Delta^\vocabsize \rightarrow \onehots$ generate samples from the MDM denoiser model.
Then we define the input to our corrector as $\ytL := (\pi \circ \x^1_\theta(\ztL), \ldots \pi \circ \x^L_\theta(\ztL)).$
We present the following objective for training a corrector alongside the standard MDM denoiser:
\begin{align}\label{eq:cmdm_objective}
    \mathcal{L}^{\mathpzc{CMDM}}_{\phi, \theta} := \E_{q_{\text{data}}}\int_0^1\E_{q_t} \sum_{\ell=1}^L\Bigg[
    \underbrace{\lambda_t \log \langle \xl_\phi(\ytL), \xl \rangle}_{\mathcal{L}^{\mathpzc{C}}} 
    + \underbrace{\frac{\atdot}{1-\at}\delta_{\zt^\ell, \mask} \log \langle \xl_\theta(\zt^{1:L}), \xl \rangle}_{\mathcal{L}^{\mathpzc{MDM}}}\Bigg]dt,
\end{align}
where $\phi$ denote the parameters of a corrector model $\x^\ell_\phi(\yL_t) \in \Delta^\vocabsize$.
The standard MDM loss in the second term ensures that we train a useful denoiser $\x_\theta$ that generates meaningful candidates for $\x_\phi$ to correct.
In addition to the MDM objective, we add a correction loss $\mathcal{L}^{\mathpzc{C}}$ to every term in the integrand.
This auxiliary loss amounts to a simple cross-entropy (CE) term, which encourages the model to identify and correct mistakes from the original denoising network.
The weighting $\lambda_t := \lambda(t) < 0$ lets us control the relative weight between the correction and diffusion losses.

\textbf{Theoretical Motivation}~
The corrector loss term in (\ref{eq:cmdm_objective}) can also be derived in a principled manner via the lens of learned predictor-corrector samplers for discrete diffusion \citep{lezama2022discrete, zhao2024informed}.
In this section, we provide a high-level overview of this argument, deferring the full derivation to Appendix \ref{appsec:predictor_corrector}.

Deviations during sampling result from a mismatch in unconditional marginals: $p_\theta(\zt) \neq  q(\zt)$.
To mitigate this, we apply the following Monte Carlo Markov Chain (MCMC), with learned model $p_\phi$:
\begin{align}\label{eq:mcmc_main}
     \ytL \sim p_\theta(\ytL \mid \ztL) ~~~\longrightarrow~~~
     \hxL \sim p_\phi(\hxL \mid \ytL) ~~~\longrightarrow~~~
     \hztL \sim q(\hztL \mid \hxL),
\end{align}
with the goal that samples from the chain are ultimately drawn from the correct marginals $q(\zt)$.

The following condition ensures that the limiting distribution of transition kernels defined by the sampling chain in (\ref{eq:mcmc_main}) indeed accomplish this goal (see proof in Appendix \ref{appsec:predictor_corrector}):
\begin{align}\label{eq:mcmc_condition_main}
    p_\phi(\xL \mid \ytL) \propto p_\theta(\ytL \mid \ztL) q_{\text{data}}(\xL).
\end{align}
Finally, optimizing $\phi$ so that the proportionality in (\ref{eq:mcmc_condition_main}) holds yields a loss term equivalent to $\mathcal{L}^{\mathpzc{C}}$.

\begin{wrapfigure}{r}{0.48\textwidth}
\vspace{-5em}
\begin{minipage}{0.48\textwidth}
\begin{algorithm}[H]
    \footnotesize
    \caption{\method~Training}
    \label{alg:train}
    \begin{algorithmic}[1]
    \STATEX \textcolor{ourgray}{\texttt{\# Differences from standard MDM training highlighted in \textbf{\textcolor{brown}{brown}}.}}
    \STATE \textbf{Input:} Training data $\mathcal{D}$, model $\x_\theta$ with parameters $\theta$, corruption process $q$, noise schedule $\alpha_t$.
    \REPEAT
        \STATE Sample $\x^{1:L}$ i.i.d. from $\mathcal{D}$
        \STATE Sample $t \sim \mathcal{U}[0, 1]$
        \STATE Compute $\at, \atdot$
        \STATE Sample $\zt^{1:L} \sim q(\z_{t}^{1:L} \mid \x)$
        \STATE Compute $\x_\theta(\zt^{1:L})$
        \STATE $\mathcal{L}_{\theta}^{\mathpzc{MDM}} \leftarrow \frac{\atdot}{1-\at}\sum_{\ell=1}^L\delta_{\zt^\ell, \mask} \log \langle \xl_\theta(\zt^{1:L}), \xl \rangle$
        \STATE \textcolor{brown}{$\ytl \leftarrow \sg(\texttt{one\_hot}(\argmax_i\x_\theta^\ell(\zt)_i)), \forall \ell$}
        \STATE \textcolor{brown}{Compute $\x_\theta(\yL_t)$}
        \STATE \textcolor{brown}{$\mathcal{L}^{\mathpzc{SC}}_{\theta} \leftarrow\frac{\atdot}{1-\at}\sum_{\ell=1}^L\log \langle \xl_\theta(\yL_t), \xl \rangle$}
        \STATE Perform gradient descent step on $\mathcal{L}_{\theta}^{\mathpzc{MDM}}$ \textcolor{brown}{$+ \mathcal{L}_{\theta}^{\mathpzc{SC}}$}
    \UNTIL{converged}
    \STATE \textbf{Return} $\theta$
    \end{algorithmic}
\end{algorithm}
\end{minipage}
\vspace{-2em}
\end{wrapfigure}
\subsection{\texorpdfstring{\method}{ProSeCo}~Design Decisions}\label{subsec:prosec_design}

\textbf{Tying Corrector and Denoiser Weights}~
In order to have a unified model that can both decode and correct, we elect to tie the weights $\phi = \theta$.
This decision has the added benefit of eliminating memory overhead for training a separate corrector model.
The error correction term in (\ref{eq:cmdm_objective}) thus becomes a self-correcting one: $\mathcal{L}^{\mathpzc{C}} \rightarrow \mathcal{L}^{\mathpzc{SC}}$

\textbf{Selecting Transformation $\pi$}~
To simplify optimization and ensure a deterministic mapping from $\ztL$ to $\yL_t$, we use $\argmax$ sampling from $\x_\theta(\ztL)$.
This aligns with how state-of-the-art MDMs, such as LLaDA \citep{nie2025large}, are used in practice.
Namely, during generation, at every iteration, each masked position from the output of $\x_\theta$ is decoded greedily.
Then using some algorithm or heuristic, we determine which positions to retain for the next round and which to keep masked.

\textbf{Setting Corrector Loss Weight $\lambda_t$}~
We found empirically that reusing the same factor from MDM was a performant strategy, i.e., $\lambda_t = \atdot / (1-\at)$.
In addition to balancing both terms in our objective, this weighting scheme is justified intuitively.
Specifically, in MDM, the weight $\atdot / (1-\at)$ discounts samples that are more heavily noised.
For training the corrector model, it is reasonable to apply a similar rationale: sequences with heavy masking which are harder to denoise will also be harder to correct and should therefore be down-weighted.

Combining the above yields our objective for effective joint denoising and self-correction training:
\begin{align}\label{eq:scmdm_objective}
    \mathcal{L}^{\mathpzc{SCMDM}}_{\theta} :=    \E_{q_{\text{data}}}\int_0^1\E_{q_t} \frac{\atdot}{1-\at}\sum_{\ell=1}^L\Bigg[
    \underbrace{
    \log \langle \xl_\theta(\ytL), \xl \rangle}_{\mathcal{L}^{\mathpzc{SC}}} + \underbrace{\delta_{\zt^\ell, \mask} \log \langle \xl_\theta(\zt^{1:L}), \xl \rangle}_{\mathcal{L}^{\mathpzc{MDM}}} \Bigg]dt.
\end{align}

\subsubsection{Training with the \texorpdfstring{$\mathpzc{SCMDM}$}{SCMDM}~Objective}
In Algorithm \ref{alg:train}, we present training for \method, which requires only minor modification to standard MDM training.
Specifically, in addition to the MDM forward pass used for (\ref{eq:mdlm_nelbo}), we sample from $\x_\theta$, perform a second forward pass, and incorporate the corrector loss.
A final modification to (\ref{eq:scmdm_objective}) is
to wrap the denoiser model's outputs in the stop-gradient operation $\sg(\cdot)$ prior to forming the corrector input, which ensures training stability.

\section{Sampling with \texorpdfstring{\method}{ProSeCo}}\label{sec:methods_sampling}
In Algorithm \ref{alg:sample}, we present sampling for \method.
Having a model that can jointly decode and correct allows us to interleave unmasking and correction steps.
The goal of corrector iterations is two-fold: they should 1) potentially update already decoded positions in $\ztL$ and 2) provide improved predictions to be used for sampling in the unmasking posterior.
To control the computation budget allocated to self-correction, we allow users to specify hyperparameters that determine how often a self-correcting loop is executed, $\omega$, and the number of steps per loop, $S$.
\begin{wrapfigure}{l}{0.51\textwidth}
\vspace{-1.7em}
\begin{minipage}[t]{0.51\textwidth}
\begin{algorithm}[H]
\footnotesize
\caption{\method~Sampling}
\label{alg:sample}
\begin{algorithmic}[1]
\STATEX \textcolor{ourgray}{\texttt{\# Differences from standard MDM sampling highlighted in \textbf{\textcolor{brown}{brown}}.}}
\STATE \textbf{Input:} Model $\x_\theta$, length $L$, unmasking steps $T$, schedule $\alpha_t$, \textcolor{brown}{self-correction loop budget $S$, corrector frequency $\omega$.}
\STATE Initialize $\z_{t(T)}^{1:L} \leftarrow \mask^{1:L}$
\FOR{$i=T$ {\bfseries to} $1$}
    \STATE $\texttt{logits} \leftarrow \x_\theta(\ztL)$
    \IFBROWN{$(T-i+1) \mod \omega == 0$}        
        \STATE \textcolor{brown}{$\ztL, \texttt{logits} \leftarrow \texttt{corrector}(\x_\theta, S, \ztL, \texttt{logits})$}
        \STATEX \textcolor{ourgray}{\texttt{// See function below}}
    \ENDIFBROWN
    \STATE $\z_{t(i-1)}^{1:L} \leftarrow \texttt{sample\_posterior}(\texttt{logits},~ \ztL, \alpha_{t(i)})$
\ENDFOR
\STATE \textbf{Return} $\texttt{sample}(\x_\theta(\z_{t(0)}^{1:L}))$
\vspace{0.1em}
\FUNCTIONBROWN{\texttt{corrector}$(\x_\theta, S, \ztL, \texttt{logits})$}
    \STATE \textcolor{brown}{\textbf{Input:} Model $\x_\theta$, self-correction budget (per step) $S$, latent sequence $\ztL$, denoising output \texttt{logits}}
    \STATE\textcolor{brown}{Initialize $\ytl \leftarrow \texttt{one\_hot}(\argmax_i\texttt{logits}_i), \forall \ell$}
    \FORBROWN{$S$ steps}
        \STATE \textcolor{brown}{$\texttt{corrector\_logits} \leftarrow \x_\theta(\yL_t)$}
        \STATE \textcolor{brown}{$\yL_t \leftarrow \texttt{sample}(\texttt{corrector\_logits})$}
    \ENDFORBROWN
    \STATEX \textcolor{ourgray}{\texttt{\# Correct unmasked positions in $\ztL$}}
    \STATE \textcolor{brown}{$\ztl \leftarrow \yl_t, \forall \ztl \neq \mask$}
    \STATE \textcolor{brown}{\textbf{Return} $\ztL$, ~\texttt{corrector\_logits}}
\ENDFUNCTIONBROWN

\end{algorithmic}
\end{algorithm}
\end{minipage}
\vspace{-1em}
\end{wrapfigure}

\vspace{-1em}
In the \textcolor{brown}{\texttt{corrector}} routine invoked during sampling, we present the self-correction procedure.
When entering this loop, we convert $\x_\theta(\ztL)$ into a corrector input via $\argmax$ sampling at every position.
Within each correcting iteration, we use a \texttt{sample} method, e.g., greedy-max decoding, to generate the next corrector input sequence from the corrector model outputs.
After the inner loop completes, unmasked positions $\zt^\ell \neq \mask$ are replaced with corresponding positions in the corrector sequence, which represents the remediation of already decoded tokens.

When operating as an `unmasking' model, an inference loop iteration is identical to that in standard MDMs: we input a partially masked sequence $\ztL$, we output predictions $\xl_\theta(\ztL)$ at every position $\ell$, and we decide some tokens to unmask via a call to a \texttt{sample\_posterior} protocol, which returns a sequence with fewer masked tokens.
For example, we can use \textit{ancestral sampling} \citep{sahoo2024diffusion}, where we replace $\x$ with $\x_\theta$ in (\ref{eq:mdlm_true_posterior}) and sample accordingly, or the \textit{confidence-based} approach proposed in \citet{nie2025large}, where proposal tokens are selected greedily from $\x_\theta(\ztL)$ and the top-$k$ confident positions are retained for the next iteration.
If correction has been applied, instead of $\x_\theta(\ztL)$, we use the corrector model's last output for the \texttt{sample\_posterior} routine, as the corrector model logits represent better informed predictions.

\section{Experiments}\label{sec:exp}
We evaluate \method~across a diverse set of tasks and for both conditional and unconditional generation.
We consistently demonstrate that our model better scales the quality-efficiency frontier and enables even further improvements via inference-compute scaling.
\begin{table*}[t!]
    \centering
    \footnotesize
    \caption{Pass@1 on code and math benchmarks.
    `Corrector Sampling' \xmark~/~\checkmark~symbol indicates whether a correction algorithm was applied.
    $^\dagger$ indicates values taken from \citet{huang2025don}.
    Other "Prior work" values obtained via evaluation with open-sourced weights.
    Best value is \textbf{bolded}.
    }
    \begin{tabular}{llccccc}
    \toprule
    & & & \multicolumn{2}{c}{\textbf{Code}} & \multicolumn{2}{c}{\textbf{Math}} \\
    \cmidrule(lr){4-5} \cmidrule(lr){6-7}
     & & \makecell{Corrector Sampling} & \makecell{HumanEval} & \makecell{MBPP} & \makecell{GSM8K} & \makecell{MATH} \\
    \midrule
    \rowcolor{lightgray!25}
    \multicolumn{7}{l}{\textit{Off-the-Shelf 8B Models}} \\
    \multirow{6}{*}{\textit{Prior Work}}
    & Llama3.1-Instruct \citep{grattafiori2024llama} & \xmark & 63.41 & 70.90 & 81.05 & 47.38 \\
    & LLaDA-Base \citep{nie2025large} & \xmark & 34.15 & 37.20 & 46.78 & 17.04 \\
    & LLaDA-Instruct \citep{nie2025large} & \xmark & 45.73 & 47.09 & 83.40 & 43.14 \\
    & ~~~~+ ReMDM \citep{wang2025remasking} & \checkmark & 43.90 & 45.50 & 83.93 & 43.76 \\
    & LLaDA1.5 \citep{zhu2025llada} & \xmark & 45.12 & 46.83 & 84.00 & 42.54 \\
    & ReMeDi-Instruct$^\dagger$ \citep{huang2025don} & \checkmark & 71.30 & 57.80 & 86.30 & 51.40 \\
    \rowcolor{lightgray!25}
    \multicolumn{7}{l}{\textit{Our SFT with LLaDA-Base 8B Model}} \\
    \multirow{3}{*}{\makecell{\textit{Baseline}}}
    & Vanilla SFT & \xmark & 58.54 & 56.88 & 88.86 & 46.60 \\
    & ~~~~+ ReMDM \citep{wang2025remasking} & \checkmark & 56.10 & 50.00 & 88.48 & 46.22 \\
    & Hybrid Mask + Uniform SFT & \checkmark & 64.02 & 56.88 & 89.01 & 48.52 \\
    \midrule
    \multirow{2}{*}{\textit{Ours}}
    & \method~SFT & \xmark & 69.51 & 57.41 & 91.36 & 51.98 \\
    & ~~~~+ \method~Max Sampling & \checkmark & \textbf{72.56} & \textbf{69.31} & \textbf{92.19} & \textbf{55.06} \\
    \bottomrule
    \end{tabular}
    \label{tab:performance_results}
\end{table*}

\subsection{Math \& Code Benchmarks}\label{subsec:exp_llada}
\textbf{Setup}~
To evaluate \method~on large MDMs, we perform supervised fine-tuning (SFT) of the LLaDA-Base 8B model \citep{nie2025large} using our training Algorithm \ref{alg:train}.
Specifically, we fine-tune for $\sim$400B tokens on a modified version of the Llama-Nemotron-Post-Training dataset \cite{bercovich2026extending}.
In this variant, the target outputs for each prompt were generated using GPT-OSS \citep{openai2025gptoss120bgptoss20bmodel}
(see Appendix \ref{appsubsec:exp_details_benchmarks}).

We then evaluate on downstream code (HumanEval \citep{chen2021evaluating} \& MBPP \citep{austin2021program}) and math benchmarks (GSM8K \citep{cobbe2021training} \& MATH \citep{hendrycks2021measuring}).
In addition to reporting results for large open-source AR and MDM models (with and without corrector mechanisms), we provide a direct comparison to our approach by applying the same SFT recipe to LLaDA-Base under two alternatives: 1) standard MDM training and 2) hybrid mask and uniform noise, similar to that proposed in GIDD \citep{von2025generalized}.

Note that for all LLaDA-based models we apply the semi-AR sampling algorithm \citep{arriola2025block} adopted by \citet{nie2025large}, where the full generation sequence $L$ is broken into blocks of size $B$ and unmasking decoding is applied block-by-block from left-to-right (see Appendix \ref{appsec:sampling_semiar}).

\textbf{Main Results}~
Table \ref{tab:performance_results} presents our main results.
Across all four benchmarks, \method~with correction sampling outperforms all diffusion baselines, including those coupled with other corrector mechanisms, and surpasses a comparably sized instruction-fine-tuned AR model (Llama3.1; \cite{grattafiori2024llama}) on three of four tasks.
Our vanilla SFT baseline already substantially improves over both LLaDA-Base and LLaDA-Instruct.
While adding uniform noise alongside masking during SFT, improves benchmark performance, \method~SFT surpasses both baselines on every benchmark, even before applying the correction sampling procedure in Algorithm \ref{alg:sample}.
These results suggest that training on model-generated errors provides benefits beyond those obtained from uniform token corruption, with \method~sampling yielding further gains across all four tasks.
We additionally explore alternative ReMDM configurations for our LLaDA SFT baseline in Appendix \ref{appsubsec:remdm_llada}, which improve its correction performance but retain a substantial gap to \method.

\textbf{Gains Without Correction Sampling}~
The gains from \method~SFT alone suggest that the self-correction objective benefits ordinary unmasking in addition to enabling correction sampling.
The two training objectives expose the shared network to complementary inputs: the standard MDM objective uses partially masked sequences with ground-truth visible tokens, whereas self-correction uses fully unmasked, model-generated sequences.
Although neither exactly reproduces ordinary inference states, which contain both masks and model-generated tokens, their combination may improve robustness to imperfect generated context.
The auxiliary objective may also be improving optimization by acting as a regularizer.

\textbf{Comparison with Extended Baseline Training}~
To train \method, we require two forward passes through the denoiser model for each training step.
Therefore, to assess whether additional standard MDM training can account for the gains of \method, we train a vanilla SFT baseline for twice the number of optimization steps.
As shown in Table \ref{tab:llada_extended_training} (Appendix \ref{app:extended_training}),
\method~outperforms the extended-training baseline on all four benchmarks implying that doubling the baseline's training steps does not close the performance gap to \method.

\textbf{Analyzing the Quality-Efficiency Trade-off}~
In Figure \ref{fig:llada_results_quality_efficiency}, we present an analysis of the quality-efficiency trade-offs for \method.
Results further to the north-west corner are desirable as they indicate better performance with a smaller number of function evaluations (NFEs).

In standard MDM, the only lever for controlling this trade-off is number of unmasking inference steps, or in other words the number of token positions generated in parallel at each decoding iteration.
For \method,  we can also control the compute via the frequency of corrector loops and number of iterative refinement steps per loop.
For each benchmark, we find that \method~can outperform the highest accuracy baseline configuration, i.e., generating one token in each iteration; depicted as the \textcolor{figgray}{gray dot}.
As depicted by the \textcolor{figgreen}{green star} marker labeled as ``\textcolor{figgreen}{Fast},''
we can strictly improve baseline accuracy with reduced NFEs by increasing the decoding parallelism (4 - 8 tokens) per unmasking step.
To compensate for this, we apply corrector loops at every 2nd decoding iteration and use up to 4 NFEs per corrector loop.
Replacing unmasking steps with correcting ones, leads to $\sim$2-4$\times$ speed-ups relative to LLaDA decoding, without sacrificing accuracy.
\method~also enables configurations that can moderately increase compute while delivering significant accuracy improvements, as depicted by the \textcolor{figorange}{orange star} marker labeled as ``\textcolor{figorange}{Balanced}'' (for best trade-off; see Appendix \ref{appsubsec:exp_details_benchmarks} for details on how this point was selected).
Finally, \method~supports even further scaling of test-time compute to attain our highest performing results depicted in the \textcolor{figblue}{blue star} markers labeled as ``\textcolor{figblue}{Max}.''

\textbf{Pareto Frontier for Parallel Decoding and Quality}~
Additionally, in Figure \ref{fig:llada_results_pareto}, we demonstrate that for standard MDMs, increasing the level of parallel decoding  significantly degrades sample quality.
In contrast, \method~models can recover from the mistakes introduced during generation and better scale the parallel decoding-quality Pareto frontier (see Appendix \ref{appsubsec:exp_results_tps} for throughput analysis).

\textbf{Ablation: Selecting Corrector Budget}~
Finally, in Figure \ref{fig:llada_results_budget} (Appendix \ref{appsubsec:ablation}), we explore the performance of various configurations of our sampling hyperparameters to provide guidance on allocating the corrector budget, as determined by frequency ($\omega$) of and number of steps ($S$) per loop.
The key takeaway, is that our model is highly robust to the choice of these hyperparameters, beating the baseline accuracy at each token parallelism level regardless of choice of $\omega$ and $S$.
Additionally for fast sampling regimes ($\text{tokens/step} \in \{4, 8\}$), we find that with more frequent corrector loops, we can overcome the drop in quality from parallel decoding and even match or beat the baseline performance achieved with $\text{tokens/step}=1$ with significant speed-up.

\textbf{Ablation: Standard MDM + Self-Correction}~
To further elucidate the importance of our proposed framework, we apply self-correction sampling from Algorithm \ref{alg:sample} to a vanilla SFT model.
In Table \ref{tab:ablation_correct_mdm} (Appendix \ref{appsubsec:ablation_correct_mdm}), we see that the standard MDM fails to correct errors, underscoring the core motivation of our work:
the baseline model never learns to change already unmasked tokens so its output at unmasked positions is generally uninformative and does not lead to error correction.

\begin{figure}
\centering
\includegraphics[width=\textwidth]{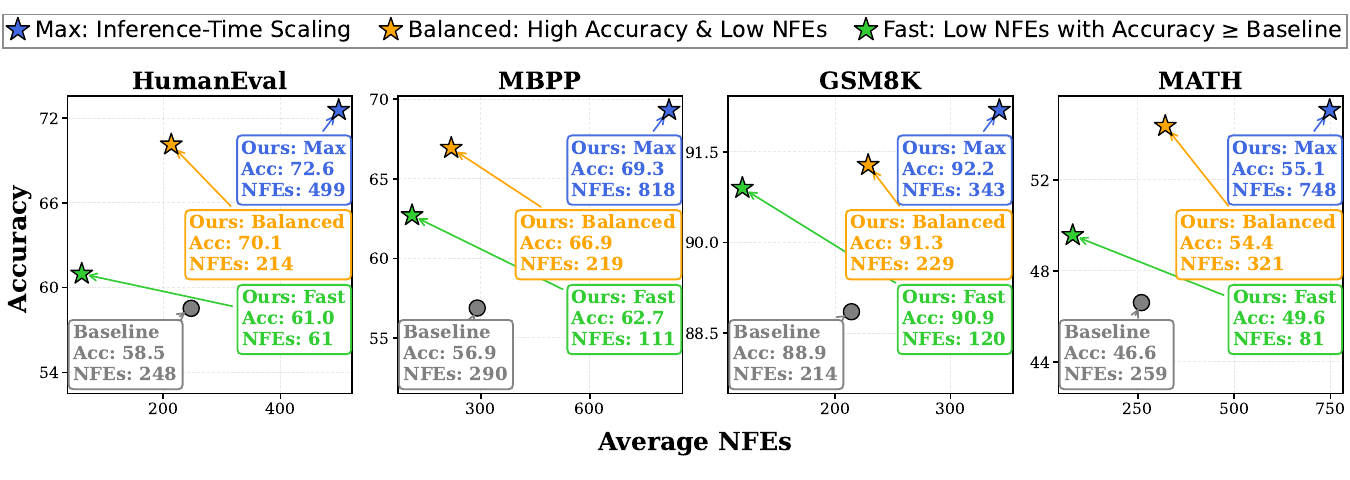}
\caption{Quality-efficiency trade-offs for \method.
Standard MDMs (Baseline; \textcolor{figgray}{gray dot}) attain best performance when decoding a single token every step.
\method~can vary number of corrector steps and attain comparable performance more efficiently (\textcolor{figgreen}{Ours: Fast; green star}), it achieve even better quality for modest increase in compute budget (\textcolor{figorange}{Ours: Balanced; orange star}), or it maximize quality by further scaling inference-time compute (\textcolor{figblue}{Ours: Max; blue star}).
}
\label{fig:llada_results_quality_efficiency}
\end{figure}

\begin{figure}
\centering
\includegraphics[width=\linewidth]{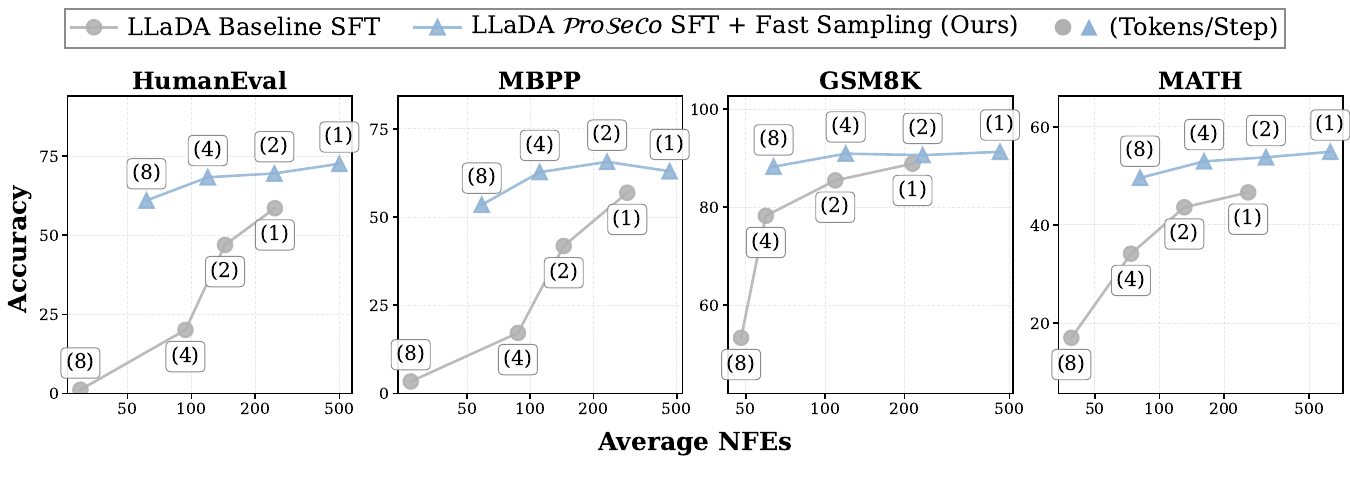}
\caption{Pareto frontier of parallel decoding and quality.
Applying a modest number of corrector steps, allows \method~models to recover from paralell decoding errors and extend this frontier.
}
\label{fig:llada_results_pareto}
\end{figure}

\subsection{Guided Molecule Design}\label{subsec:exp_qm9}
In the context of guided generation, often when guidance strength is increased, model samples collapse.
Our hypothesis is that \method~can recover from these errors, thereby improving the guided generation trade-off of maximizing a property of interest while producing a diverse set of samples.

\textbf{Setup}~
We train models on string representations of molecules known as SMILES \citep{weininger1988smiles} from the QM9 dataset \citep{ramakrishnan2014quantum, ruddigkeit2012enumeration} (see Appendix \ref{appsubsec:exp_details_guidance}).
We then apply a discrete classifier-free-guidance (CFG) algorithm \citep{schiff2024simple} with varying unmasking budgets $T$ and guidance strength $\gamma$.
We measure the number of generated sequences that are valid (can be parsed by RDKit library \citep{landrum2013rdkit}), unique, and novel (do not appear in the QM9 dataset) as the metric for diverse, high quality samples, and for the novel sequences, we compute the mean property value as the metric for property maximization.
We perform this experiment for two properties: ring count and drug-likeness (QED; \cite{bickerton2012quantifying}).
We compare \method~to an AR model, a diffusion model trained with uniform categorical noise (UDLM; \cite{schiff2024simple}), and a standard masked diffusion model (MDLM; \cite{sahoo2024simple}).
We also compare to the three remasking strategies ReMDM \citep{wang2025remasking}, ReMeDi \citep{huang2025don}, and GStar \citep{meshchaninov2025guided}.

\textbf{Results}~
In Figure \ref{fig:qm9}, we present the guidance results.
Points further north-east are preferable, as they represent property maximization without sacrificing sample diversity and quality.
For both properties of interest, \method~pushes the Pareto frontier in the desired direction.
This is particularly stark for experiments where we maximize the ring count property (left hand side of Figure \ref{fig:qm9}).

\begin{figure}[t]
    \centering
    \begin{minipage}{.3\linewidth}\centering
    \includegraphics[width=1.1\linewidth]{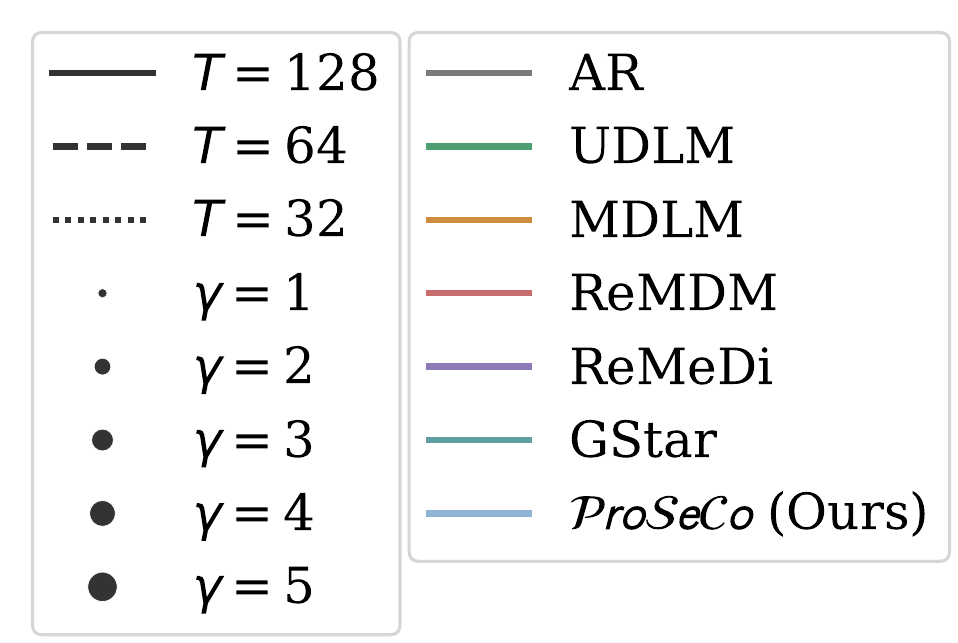}
    \end{minipage}
    \hfill
    \begin{minipage}{.33\linewidth}\centering
    \includegraphics[width=1.0\linewidth]{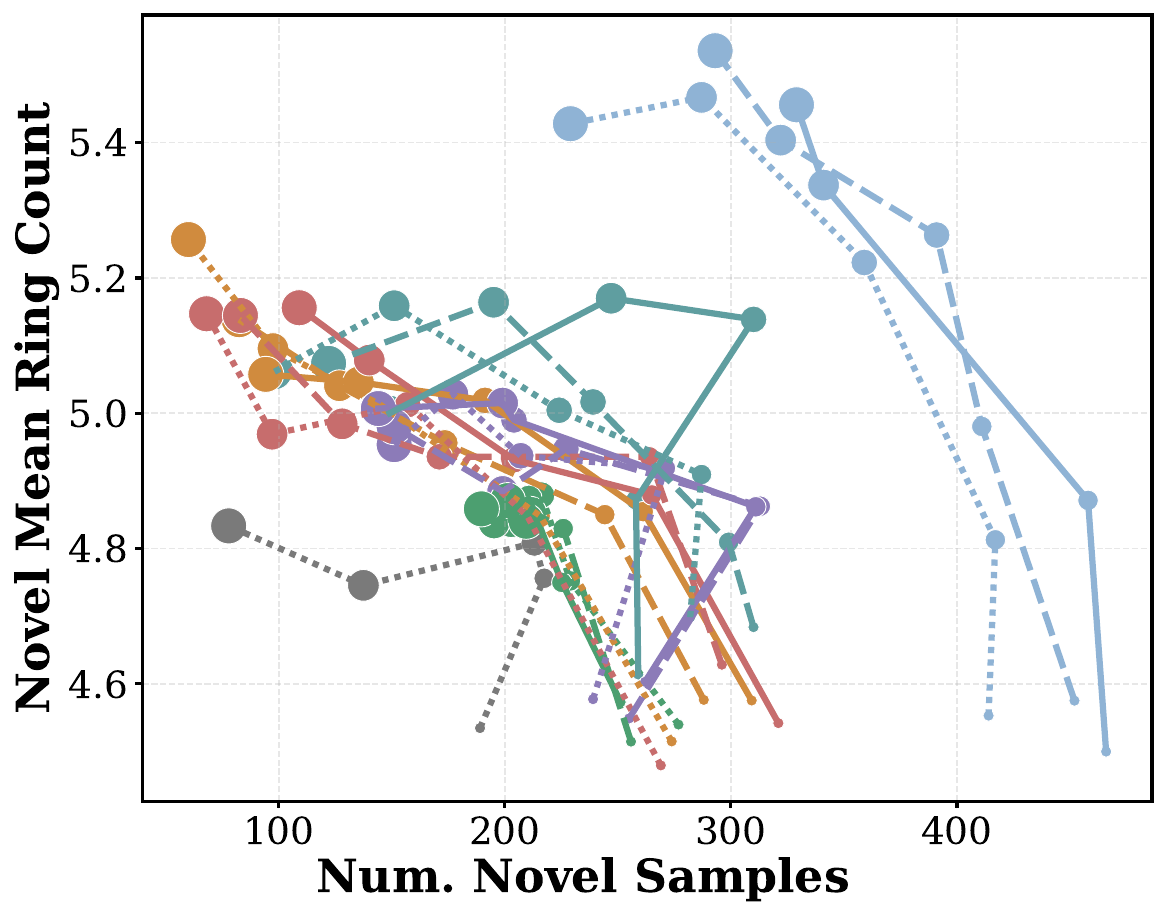}
    \end{minipage}
    \begin{minipage}{.33\linewidth}\centering
    \includegraphics[width=1.0\linewidth]{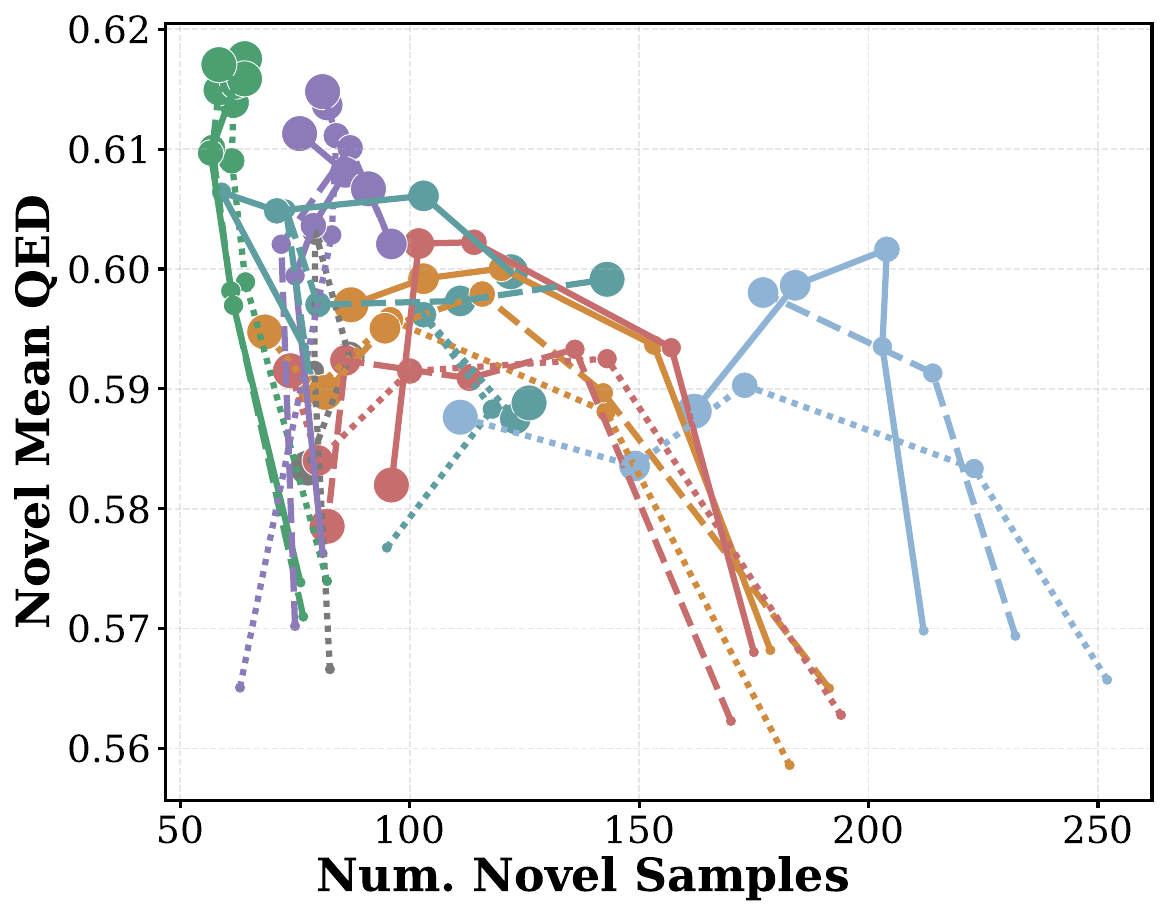}
    \end{minipage}
    \caption{\method~better navigates the novelty-property maximization Pareto frontier.
    Novel samples ($x$-axis) vs. mean property value of novel samples ($y$-axis) for classifier-free guidance \citep{schiff2024simple}, with varying unmasking steps $T$ (line style) and guidance strength $\gamma$ (marker size).
    \textit{(Left)} Maximizing the \textit{ring count} property.
    \textit{(Right)} Maximizing the \textit{drug likeness (QED)} property.
    }
    \label{fig:qm9}
\end{figure}

\begin{table*}[t]
    \centering
    \footnotesize
    \caption{
    Unconditional generation sample quality for models trained on OpenWebText.
    $^\dagger$ Values reported in \citet{kim2025fine}.
    $^\$$ Values reported in \citet{wang2025remasking}.
    $^*$ Values obtained with open-sourced weights; sequence length reduced to 512 to match the model trained in \citet{von2025generalized}.
    }
    \begin{tabular}{lrcccc|cccc|cccc}
    \toprule
    & & \multicolumn{4}{c}{MAUVE ($\uparrow$)} & \multicolumn{4}{c}{Gen. PPL ($\downarrow$)} & \multicolumn{4}{c}{Entropy ($\uparrow$)} \\
    \midrule
        \multicolumn{2}{l}{Data} & \multicolumn{4}{c|}{1.00} & \multicolumn{4}{c|}{14.8} & \multicolumn{4}{c}{5.44} \\
    \midrule
        \multicolumn{2}{l}{AR ($T=1024$)$^\$$} & \multicolumn{4}{c|}{0.76} & \multicolumn{4}{c|}{12.1} & \multicolumn{4}{c}{5.22} \\
    \midrule
    \rowcolor{lightgray!25}
    \multicolumn{2}{r}{$T=$} & $128$ & $256$ & $512$ & $1024$ & $128$ & $256$ & $512$ & $1024$ & $128$ & $256$ & $512$ & $1024$ \\
    \multicolumn{2}{l}{MDLM$^\$$ \citep{sahoo2024simple}}
    & 0.015 & 0.023 & 0.031 & 0.042 & 61.5 & 55.8 & 53.0 & 51.3 & 5.52 & 5.49 & 5.48 & 5.46 \\
    \multicolumn{2}{l}{ReMDM$^\$$ \citep{wang2025remasking}}
    & 0.057 & 0.216 & 0.350 & 0.403 & 42.5 & 30.5 & 21.1 & 28.6 & 5.43 & 5.34 & 5.21 & 5.38 \\
    \multicolumn{2}{l}{PRISM$^\dagger$ \citep{kim2025fine}} 
    & 0.118 & 0.294 & 0.423 & 0.527 & 21.5 & 18.0 & 16.4 & 15.3 & 5.18 & 5.15 & 5.12 & 5.10 \\
    \multicolumn{2}{l}{GIDD$^*$ \citep{von2025generalized}} & 0.268 & 0.284 & 0.334 & 0.356 & 95.1 & 80.5 & 76.9 & 76.1 & 5.00 & 4.94 & 4.94 & 4.94 \\
    \multicolumn{2}{l}{\method~\textit{(Ours)}} & 
    \textbf{0.295} & \textbf{0.557} & \textbf{0.597} & \textbf{0.604} & 23.1 & 16.5 & 13.2 & 10.9 & 5.45 & 5.39 & 5.29 & 5.22 \\
    \bottomrule
    \end{tabular}
    \label{tab:owt}
\end{table*}

\textbf{Ablation on Choice of $\lambda_t$}
We ablate the effect of our choice of $\lambda_t$ by repeating the guidance experiment for the ring count property with fixed $\lambda \in \{0.1, 1, 10\}$ and a time-varying weighted $ \lambda_t \in \{ 0.1\cdot \tfrac{\atdot}{(1-\at)}, 1\cdot \tfrac{\atdot}{(1-\at)}, 10\cdot \tfrac{\atdot}{(1-\at)}\}.$
We report results in Figure \ref{fig:lambda_ablation} (Appendix \ref{appsubsec:ablation_lambda}), where we find that our design choice of including $\tfrac{\atdot}{(1-\at)}$ in the corrector loss improves performance and that our model is robust to different scaling factors applied to this weighting scheme.

\subsection{Unconditional Text Generation}\label{subsec:exp_text}
\textbf{Setup}~
We train \method~from scratch on the OpenWebText (OWT; \cite{Gokaslan2019OpenWeb}) dataset for 1M steps (see Appendix \ref{appsubsec:exp_details_text}).
We then generate 5000 samples consisting of $L=1024$ tokens, for varying sampling budgets $T$.
We compute MAUVE \citep{pillutla2021mauve} and report the perplexity under the GPT2-Large model \citep{radford2019language}, to also measure quality, and average sequence entropy, to reflect diversity \citep{zheng2024masked}.
We compare against an AR model and MDLM \citep{sahoo2024simple} without and with correctors: ReMDM \citep{wang2025remasking} and PRISM \citep{kim2025fine}.
We also compare to the hybrid mask and uniform noise model GIDD \citep{von2025generalized}.

For \method, we apply 3 corrector steps after each unmasking step.
For parity with the compute budget of other methods, we reduce the number of unmasking steps by a factor of 4 per sampling budget $T$.
For this experiment, we use slightly modified training and sampling (see Appendix \ref{appsubsec:exp_details_text}).

\textbf{Results}~
In Table \ref{tab:owt}, we see that across budgets, \method~either significantly outperforms or matches baseline methods.
Notably, even with just 256 steps, \method~attains comparable sample quality to using PRISM with 2$\times$ or ReMDM with 4$\times$ the inference budget.

\section{Related Works}\label{sec:related_works}
\textbf{Discrete Diffusion}~
The seminal work of D3PM \citep{austin2021structured} laid the foundation for adapting diffusion to discrete data.
Some works extended this paradigm via the formalism of continuous-time Markov chains \citep{campbell2022continuous, loudiscrete}.
However, our method is more in line with the continuous-time extensions of the variational inference perspective \citep{ou2024your, sahoo2024diffusion, shi2024simplified}.
Previous efforts in this vein have relied on categorical uniform noise corruptions \citep{schiff2024simple, von2025generalized} to alleviate the locked-in decoded tokens limitation of MDMs.
Recently LLaDA 2.1 \citep{bie2026llada2} scaled up the GIDD framework to 16B-100B parameters by fine-tuning with mixed mask and uniform noise.
In contrast, our work which uses the model’s own unmasking output to train for correction, more closely aligns training and inference distributions, since during generation mistakes will not resemble uniform noise, but rather will come from the model’s own outputs.
Our approach thus uses a more ‘informed’ noise relative to the mask-uniform hybrid approach.
\nocite{sohl2015deep,song2019generative,dhariwal2021agn,ho2020denoising,kingma2013auto,kuleshov2013fast,wang2023infodiff,li2022diffusion,dream2025,nie2025large,shi2024simplified,arriola2025block,wang2025d2,schiff2024simple}

\textbf{Self-conditioning}~
In self-conditioning \citep{chen2022analog},
model predictions of the clean data serve as auxiliary conditioning for subsequent denoising predictions.
Both self-conditioning and \method~exploit the model's own predictions, but assign them different roles.
In \method, the predicted clean tokens form a fully unmasked input to the shared network, which is explicitly trained to reconstruct the ground-truth sequence through a self-correction
loss.
At inference time, this capability allows correction steps to directly revise previously decoded tokens.
Thus, our approach uses model predictions as inputs to an explicit correction task, rather than solely as auxiliary conditioning for denoising.

\textbf{Step Unrolling}~
Also related to our work is training on unrolled predictions from the denoising trajectory \citep{savinov2021step}, where the model is trained on its own outputs of less noisy latent sequences to more closely simulate the distribution seen during generation.
In contrast, our method uses predictions of clean data, not unrolled trajectories of partially masked sequences.

\textbf{Corrector Methods}~
Several works used a predictor-corrector framework to improve sample quality, where, following an unmasking predictor step, a corrector step is applied to remask decoded positions \citep{campbell2022continuous, gat2024discrete, peng2025path, wang2025remasking}.
In contrast to these training free methods, other works propose to train an additional head to predict incorrect positions that should be remasked \citep{huang2025don, kim2025fine, lezama2022discrete, liu2026teach, meshchaninov2025guided}.
More related to our work is \citet{zhao2024informed}, which predicts corrections to already decoded tokens.
However this method relies on a distinct Hollow Transformer backbone \citep{sun2022score}.
This severely limits its application to fine-tuning of MDMs, e.g., LLaDA, pre-trained with the standard Transformer backbones.

\section{Discussion \& Conclusion}\label{sec:discussion}
In this work, we presented a framework for jointly training a diffusion model to unmask and self-correct.
We enable and take advantage of this new ability via minimal and straightforward modifications to standard MDM training and sampling algorithms.
Evaluating on conditional and unconditional generation, across various model sizes, we demonstrated that our method consistently outperforms vanilla MDMs and alternative corrector methods both in terms of speed-quality tradeoffs and in the ability to further scale inference-time compute for improved generation.

\textbf{Limitations}~
The key drawback of our work is the added computational cost of the second forward pass during training, especially in contrast to inference-time only schemes, e.g., \citet{wang2025remasking}.
However, the empirical results demonstrate that this trade-off of train-time compute is well worth the performance and efficiency gains achieved on downstream evaluations.
Moreover, our comparison to vanilla masked diffusion training with 2$\times$ the training budget demonstrates that our method provides gains that cannot be simply accomplished with further training.

\textbf{Future Directions}~
In follow up work, we plan to explore the disentangling of the corrector and unmasking models via weight untying or with completely separate neural network backbones for each model.
Additionally, while we present a performant sampling algorithm, the ability to correct mistakes opens up the design space to more sophisticated schemes of jointly using corrector and unmasking steps, which we leave to future work to explore.

\section*{Acknowledgments}
This work was partially funded by the National Science Foundation under award CAREER 2145577, and by the National Institute of Health under award MIRA R35GM151243.
Marianne Arriola is supported by a NSF Graduate Research Fellowship under award DGE-2139899 and a Hopper-Dean/Bowers CIS Deans Excellence Fellowship.
\clearpage
\bibliographystyle{plainnat}
\bibliography{refs}

\clearpage
\allowdisplaybreaks
\appendix
\setcounter{tocdepth}{2}
\tableofcontents
\newpage

\section{Deriving \texorpdfstring{\method}{ProSeCo}~from MCMC Predictor-Corrector}\label{appsec:predictor_corrector}
We now show how we can derive the corrector loss term $\mathcal{L}^{\mathpzc{C}}$ from the objective in (\ref{eq:cmdm_objective}) by leveraging formulations from predictor-corrector sampling for continuous-time Markov chains \citep{campbell2022continuous}.
In particular, we start from the perspective that sampling errors are a byproduct of the fact that the posteriors of the learned reverse process do not align with those from predefined corruption process $p_\theta(\zs \mid \zt) \neq q(\zs \mid \zt)$.
This issue arises because of two reasons, 1) the model may not fully capture the exact marginals of each token position, and 2) the inherent parameterization of our reverse process as factorized independently across the sequence length dimension does not account for cross-token dependencies \cite{ben2025accelerated}.
These deviations compound along a generation path and ultimately lead to marginals that are significantly divergent: $p_\theta(\zt) \neq q(\zt)$.
In score-matching generative modeling for continuous data \citep{song2020score}, this mismatch can be mitigated by Langevin Monte Carlo Markov chain (MCMC) iterations.
Unfortunately in the discrete data setting, we do not have access to a continuous score function that drives the predictor-corrector sampler proposed in \citet{song2020score}.

Crucially, \citet{campbell2022continuous} demonstrate that for discrete data we can also apply corrector steps and produce a MCMC that yields $q(\zt)$ as its stationary distribution.
Interleaving corrector MCMC steps with predictor (unmasking) ones enables us to recover from deviations before proceeding along the generation trajectory.
While \citet{campbell2022continuous} observe better sample quality with their proposed `forward-backward' corrector scheme, follow up works \citep{lezama2022discrete, zhao2024improving} note that this methodology can be further improved by learning the reverse transition kernel for the corrector MCMC, leading to more `informed' corrector steps.

At a high level, our derivation below proceeds as follows.
We first propose a corrector sampling chain.
We then prove that if the corrector distribution satisfies a given functional form then we will achieve our goal of having the samples from the MCMC corrector step $\sim q(\ztL)$.
Finally, we enforce this condition by minimizing a divergence between the corrector distribution induced by our model and the desired form.
This training recovers the corrector loss proposed in Section \ref{sec:methods}.

\textbf{Corrector MCMC}~
In the vein of learned predictor-corrector works \citep{lezama2022discrete}, we define an MCMC corrector step by the following chain:
\begin{align}\label{eq:mcmc}
    \begin{split}
        \text{1) Sample } \ytL &\sim p_\theta(\ytL \mid \ztL), \\
        \text{2) Sample } \hxL &\sim p_\phi(\hxL \mid \ytL), \\
        \text{3) Sample } \hztL &\sim q(\hztL \mid \hxL),
    \end{split}
\end{align}
where $p_\theta(\ytL \mid \ztL) = \prod_{\ell=1}^L\cat(\ytl; \xl_\theta(\ztL))$ is the distribution defined by parallel decoding from the output of our denoising network $\x_\theta$ and $q(\hztL \mid \hxL)$ follows the factorized masked diffusion forward process, introduced in Section \ref{sec:background}.
Additionally, $p_\phi$ represents a corrector model with parameters $\phi$.

Intuitively the sampling chain above represents the following correction mechanism: 1) we sample from the denoising model's prediction given the current latent variable, $\ytL \sim p_\theta(\ytL \mid \ztL)$, then 2) given this sample, the corrector predicts the clean data distirbution, and we sample from this prediction, $\hxL \sim p_\phi(\hxL \mid \ytL)$, and finally 3) we re-noise the sample from the corrector output according to the predefined noising process, $\hztL \sim q(\hztL \mid \hxL)$.
The goal of this chain is that the re-noised latent variable $\hztL$ more closely resembles draws from the true distribution $q(\ztL)$.

\textbf{Detailed Balance Condition}~
We now describe a condition on $p_\phi$ that will achieve this goal.
\begin{proposition}
The following form will ensure the detailed balanced condition, with $q(\ztL)$ as the stationary distribution:
\begin{align}\label{eq:corrector_condition}
p_\phi(\xL \mid \ytL) q(\ztL \mid \xL) &= \frac{p_\theta(\ytL \mid \ztL) q(\ztL \mid \xL) q_{\text{data}}(\xL)}{Z(\ytL)},
\end{align}
where $Z(\ytL)$ is a normalization constant representing the marginal probability of $\ytL$.
\end{proposition}
\begin{proof}
We prove this using a similar argument to that presented in \citet{lezama2022discrete}.
Specifically, let $R(\hztL \mid \ztL)$ denote the transition kernel defined by the sampling chain in (\ref{eq:mcmc}), then
\begin{align}
    &q(\ztL) R(\hztL \mid \ztL) 
    = q(\ztL) \sum_{\ytL}\sum_{\xL}p_\phi(\xL \mid \ytL)q(\hztL \mid \xL) p_\theta(\ytL \mid \ztL) \quad\textit{\footnotesize{\{From (\ref{eq:mcmc})}\}} \nonumber \\
    &= q(\ztL) \sum_{\ytL}\sum_{\xL}\left( \frac{p_\theta(\ytL\mid\hztL)q(\hztL \mid \xL)q_{\text{data}}(\xL)}{Z(\ytL)} \right) p_\theta(\ytL \mid \ztL) \quad\quad\quad\quad\textit{\footnotesize{\{From (\ref{eq:corrector_condition})\}}} \nonumber \\
    &= q(\ztL) \sum_{\ytL}\left(\frac{p_\theta(\ytL\mid\hztL)q(\hztL)}{Z(\ytL)} \right) p_\theta(\ytL \mid \ztL) \quad\quad\quad\quad\quad\quad\quad\quad\quad\quad\quad\textit{\footnotesize{\{Marginalize $\x$\}}} \nonumber \\
    &=q(\hztL) \sum_{\ytL}\left(\frac{p_\theta(\ytL \mid \ztL)q(\ztL)}{Z(\ytL)} \right)p_\theta(\ytL\mid\hztL) \quad\quad\quad\quad\quad\quad\quad\textit{\footnotesize{\{Note `role change' of $\zt$ \& $\hzt$\}}} \nonumber \\
    &=q(\hztL) \sum_{\ytL}\sum_{\xL}\left(\frac{p_\theta(\ytL \mid \ztL)q(\ztL \mid \xL)q_{\text{data}}(\xL)}{Z(\ytL)} \right)p_\theta(\ytL\mid\hztL) \nonumber \\
    &= q(\hztL) \sum_{\ytL}\sum_{\xL}p_\phi(\xL \mid \ytL)q(\ztL \mid \xL)p_\theta(\ytL\mid\hztL) \quad\quad\quad\quad\quad\quad\quad\quad\quad\quad\textit{\footnotesize{\{From (\ref{eq:corrector_condition})\}}} \nonumber \\
    &= q(\hztL) R(\ztL \mid \hztL). \quad\quad\quad\quad\quad\quad\quad\quad\quad\quad\quad\quad\quad\quad\quad\quad\quad\quad\quad\quad\quad\quad\quad\quad\quad\quad\textit{\footnotesize{\{From (\ref{eq:mcmc})\}}}
    \nonumber
\end{align}
\end{proof}

\textbf{Maximum Likelihood Corrector Training}~
Having shown that the condition in (\ref{eq:corrector_condition}) is sufficient for achieving the desired limiting distribution $q(\ztL)$ for the corrector MCMC, we can cast the corrector training as ensuring that $p_\phi(\xL \mid \ytL)q(\ztL \mid \xL) \propto p_\theta(\ytL \mid \ztL) q(\ztL \mid \xL) q_{\text{data}}(\xL)$ \citep{lezama2022discrete}.
We can achieve this by minimizing the following objective:
\begin{align}\label{eq:corrector_kl}
\begin{split}
    \KL[&p_\theta(\ytL \mid \ztL) q(\ztL \mid \xL) q_{\text{data}}(\xL) \vert\vert p_\phi(\xL \mid \ytL)q(\ztL \mid \xL)Z(\ytL)] \\
    &= -\E_{\xL \sim q_{\text{data}}(\xL)}\E_{\ztL \sim q(\ztL \mid \xL)}\E_{\ytL \sim p_\theta(\ytL \mid \ztL)}[\log p_\phi(\xL \mid \ytL)] + C,
\end{split}
\end{align}
where $C$ represents factors that do not depend on $\phi$.
Minimizing this objective ensures that the distribution over latents $\ztL$ induced by the corrector steps aligns with the true marginal distributions.

\textbf{Parameterizing the Corrector Model}~
Finally, similar to the denoising model, we parameterize the corrector to be independent across the sequence length dimension given its inputs, as follows:
\begin{align}\label{eq:corrector_parameterization}
    p_\phi(\xL \mid \ytL) = \prod_{\ell=1}^L p_\phi(\xl \mid \ytL), \quad \text{with }~ p_\phi(\xl \mid \ytL) = \cat(\xl; \xl_\phi(\ytL)),
\end{align}
where $\xl_\phi(\ztL) \in \Delta^\vocabsize$ for all $\ell \in \{1, \ldots, L\}$ is the output of a corrector model network.
This independent factorization for the corrector model enables us to maintain parallel generation.

Using this factorized parameterization and replacing the sampling of $\ytL$ with the transformation $\pi \circ \xL_\theta$ defined in Section \ref{subsec:methods_corrector_loss}  recovers the $\mathcal{L}^\mathpzc{C}$ corrector loss term from the objective in (\ref{eq:cmdm_objective}):
\begin{align*}
    \mathcal{L}^\mathpzc{C} = \log \langle \xl_\phi(\ytL), \xl \rangle.
\end{align*}

\section{Sampling with \texorpdfstring{\method}{ProSeCo}~Semi-Autoregressively}\label{appsec:sampling_semiar}
In Algorithm \ref{alg:sample_bd3lm}, we present a modified version of our sampling method from Algorithm \ref{alg:sample}, which accommodates the block autoregressive decoding proposed in \citet{arriola2025block} and adopted by LLaDA \citep{nie2025large}.
Given that we applied block AR decoding to the LLaDA models, the implementation provided in Algorithm \ref{alg:sample_bd3lm} assumes full bidirectional attention is applied across the entire sequence at every forward pass, as in LLaDA, and is not written to support key-value (KV) caching.
However, this algorithm can be easily adapted to support efficient KV caching proposed in \citet{arriola2025block}.

Notably, for \method~with semi-AR decoding, at every correction iteration, clean tokens in the current block and all previously decoded blocks can be adapted.

\begin{algorithm}[!ht]
    \small
    \caption{\method~Sampling Block Autoregressive}
    \label{alg:sample_bd3lm}
    \begin{algorithmic}[1]
    \STATEX \textcolor{ourgray}{\texttt{\# Assumes full bidirectional attention without KV-caching, as in LLaDA.}}
    \STATEX \textcolor{ourgray}{\texttt{\# Differences to standard MDM with block AR decoding highlighted in \textbf{\textcolor{brown}{brown}}.}}
    \STATE \textbf{Input:} Model $\x_\theta$, length $L$, block size $B$, unmasking steps $T$, noise schedule $\alpha_t$, \textcolor{brown}{self-correction budget (per step) $S$, corrector frequency $\omega$.}
    \STATE Initialize $\z_{t(T)}^{1:L} \leftarrow \mask^{1:L}$
    \FOR {$b=1$ {\bfseries} to $(L / B)$}
    \FOR{$i=T$ {\bfseries to} $1$}
        \STATE $\texttt{logits} \leftarrow \x_\theta(\ztL)$
        \IFBROWN{$(T-i+1) \mod \omega == 0$}
        \STATE \textcolor{brown}{$\zt^{1:L},~ \texttt{logits} \leftarrow \texttt{corrector}(\x_\theta, \zt^{1:L}, S)$}
        \ENDIFBROWN
        \STATE \texttt{logits}$^\ell$ $\leftarrow -\infty, \forall \ell \in [1, (b-1)\cdot B] \cup [b\cdot B +1, L]$
        \STATE $\z_{t(i-1)}^{1:L} \leftarrow \texttt{sample\_posterior}(\texttt{logits},~ \ztL, \alpha_{t(i)})$
    \ENDFOR
    \STATE $\z_{t(0)}^{\ell} \leftarrow \texttt{sample}(\x^\ell_\theta(\z_{t(0)}^{1:L})), \forall \ell \in [1+ (b-1)\cdot B, b\cdot B]$
    \STATE $\z_{t(T)}^{1:L} \leftarrow \z_{t(0)}^{1:L}$
    \ENDFOR
    \STATE \textbf{Return} $\z_{t(0)}^{1:L}$
    \end{algorithmic}
\end{algorithm}

\section{Additional Experimental Details}\label{appsec:exp_details}
\subsection{Math \& Code Benchmarks}\label{appsubsec:exp_details_benchmarks}

\textbf{Dataset}~
Our SFT dataset is a modified version of the Llama-Nemotron-Post-Training dataset \cite{bercovich2026extending}, comprising approximately 32 million samples across four core domains: mathematical reasoning, coding, science, and instruction following. The distribution is heavily skewed toward math (66.84\%) and coding (30.62\%), with smaller allocations for science (2.15\%), instruction following (0.17\%), chat (0.12\%), and safety (0.10\%). All target outputs were generated using GPT-OSS \citep{openai2025gptoss120bgptoss20bmodel} in high-reasoning mode.

\textbf{Training Hyperparameters}~
We fine-tuned the LLaDA-Base 8B model for $\sim$400 B tokens, which amounts to 3 epochs of training on our SFT dataset.
We train with a batch size of 1024.
For learning rate we linearly warm-up for 1000 gradient steps until a maximum learning rate of 2.5$\mathrm{e}^{-5}$.
After this peak, we apply cosine decay until a minimum learning rate of 2.5$\mathrm{e}^{-7}$.
We use the \textsc{ADAM-w} optimizer \citep{kingma2014adam} with beta parameters (0.9, 0.999).
Finally, during training it is common to set a $\texttt{min\_t} > 0$ value which biases the sampling of timesteps away from uniform over the unit interval, by shifting all samples to be in the range $[\texttt{min\_t} ,1]$.
For example, in works such as \citet{sahoo2024diffusion}, this value is set to 1$\mathrm{e}^{-3}$.
For our SFT experiments, we found that biasing towards heavier masking during training led to improved performance, hence we set \texttt{min\_t} to 1$\mathrm{e}^{-1}$.

\textbf{Compute resources}~
We fine-tuned both our model and the baseline using 256 NVIDIA B200 GPUs, and conducted all evaluation runs on 8 NVIDIA H100 GPUs.

\textbf{Evaluation}~
For evaluation, we rely on the Nemo Skills library.
We evaluate all models with batch size 1.
This is to mitigate varying padding lengths based on prompt size variation and to enable effective use of early-stopping whenever the \texttt{[EOS]} token is generated.
All evaluated models use a maximum generation length of 1024.

We evaluate 4 benchmarks: HumanEval \citep{chen2021evaluating} and MBPP \citep{austin2021program}, GSM8K \citep{cobbe2021training} and MATH \citep{hendrycks2021measuring}.
We use 0-shot evaluation for each benchmark.

For LLaDA models, we use a semi-AR decoding scheme  \citep{arriola2025block}, as in \citet{nie2025large}, with default block size of 32.

\textbf{Prompts and Evaluation Templates}~
Below, we detail the templates used for evaluating models in Table \ref{tab:performance_results}.
These prompts were used for all models, except the off-the-shelf LLaDA-Base model.
For HumanEval and MBPP the placeholder \textcolor{bloodred}{\texttt{\{question\}}} is replaced by the prompt from the benchmark dataset.
Similarly, for GSM8K and MATH, the placeholder \textcolor{bloodred}{\texttt{\{problem\}}} is replaced by the prompt from the benchmark dataset.
Additionally, we use the default chat templates from the Nemo Skills library.

\begin{figure}[!ht]
    \centering
    \small
    \fbox{\begin{tabular}{lc}
    \makecell{\textbf{HumanEval / MBPP} \\ \textbf{Prompt:}}&
    \begin{minipage}{0.7\textwidth}
    \texttt{Here is a problem for which you need to generate/complete code:}\\
    \textcolor{bloodred}{\texttt{\{question\}}}
    \\
    \\
    \texttt{Please continue to complete the function with python programming language. You are not allowed to modify the given code and do the completion only.}\\
    \texttt{The solution should be in the following format:}\\
    \texttt{```python}\\
    \texttt{\# Your code here}\\
    \texttt{'''}
    \end{minipage} \\ \midrule
    \makecell{\textbf{GSM8K / MATH} \\ \textbf{Prompt:}}&
    \begin{minipage}{0.7\textwidth}
    \texttt{Solve the following math problem. Show your step-by-step reasoning, then provide the final answer inside  \textbackslash boxed\{\}}
    \\
    \\
    \textcolor{bloodred}{\texttt{\{problem\}}}
    \end{minipage}
    \end{tabular}}
    \label{fig:prompts}
\end{figure}

Note that for LLaDA-Base, we do not provide any additional prompt text, we only use the benchmark question as given.

\textbf{Sampling Hyperparameters}~
When evaluating \method~models, we explore different configurations of unmasking and corrector budgets.
In Table \ref{tab:llada_hyperparams}, we detail the sampling hyperparameters used to generate the best performing results for \method~Max reported in Table \ref{tab:performance_results} and Figure \ref{fig:llada_results_quality_efficiency}; see Appendix \ref{appsubsec:correction_configs} for more details.
Note that $T$ represents a maximum unmasking budget, since we apply early stopping on the \texttt{[EOS]} token.
Additionally, $S$ represents a maximum corrector budget per correction loop, because we break the loop iterations if a corrector sequence does not change between rounds.

For Figure \ref{fig:llada_results_pareto} (and Figure \ref{fig:pareto_tps}, below), we use use the Fast configuration, applying a corrector loop every 4 unmasking steps, with a maximum of 2 corrector steps per loop.

\begin{table*}[ht!]
    \centering
    \small
    \caption{Sampling hyperparameters for results attained with \method~and reported in Figure \ref{fig:llada_results_quality_efficiency}.
    }
    \begin{tabular}{llcccc}
    \toprule
    & & \textbf{HumanEval} & \textbf{MBPP} & \textbf{GSM8K} & \textbf{MATH} \\
    \midrule
    \rowcolor{lightgray!25}
    \multicolumn{6}{l}{\textit{Baseline (Vanilla SFT)}} \\
    & Accuracy (\%) & 58.54 & 56.88 & 88.86 & 46.60 \\
    & Average NFEs & 247.8 & 290.3 & 214.2 & 259.4 \\
    & Generation length & 1024 & 1024 & 1024 & 1024 \\
    & Block length & 32 & 32 & 32 & 32 \\
    & Demasking Tokens/Step & 1 & 1 & 1 & 1 \\
    \midrule
    \rowcolor{lightgray!25}
    \multicolumn{6}{l}{\textit{\method~Max}} \\
    & Accuracy (\%) & 72.56 & 69.31 & 92.19 & 55.06 \\
    & Average NFEs & 499.2 & 818.3 & 342.5 & 747.5 \\
    & Generation length & 1024 & 1024 & 1024 & 1024 \\
    & Block length & 32 & 32 & 32 & 32 \\
    & Demasking Tokens/Step & 1 & 1 & 1 & 1 \\
    & Corrector frequency $\omega$ & 2 & 1 & 2 & 2 \\
    & Maximum corrector steps per loop $S$ & 4 & 8 & 1 & 8 \\
    \rowcolor{lightgray!25}
    \multicolumn{6}{l}{\textit{\method~Balanced}} \\
    & Accuracy (\%) & 70.12 & 66.93 & 91.28 & 54.36 \\
    & Average NFEs & 213.9 & 218.7 & 228.7 & 321.2 \\
    & Generation length & 1024 & 1024 & 1024 & 1024 \\
    & Block length & 32 & 32 & 32 & 32 \\
    & Demasking Tokens/Step & 4 & 4 & 4 & 4 \\
    & Corrector frequency $\omega$ & 1 & 1 & 1 & 1 \\
    & Maximum corrector steps per loop $S$ & 8 & 8 & 8 & 8 \\
    \rowcolor{lightgray!25}
    \multicolumn{6}{l}{\textit{\method~Fast}} \\
    & Accuracy (\%) & 60.98 & 62.70 & 90.90 & 49.56 \\
    & Average NFEs & 61.4 & 110.9 & 119.5 & 81.2 \\
    & Generation length & 1024 & 1024 & 1024 & 1024 \\
    & Block length & 32 & 32 & 32 & 32 \\
    & Demasking Tokens/Step & 8 & 4 & 4 & 8 \\
    & Corrector frequency $\omega$ & 2 & 2 & 2 & 2 \\
    & Maximum corrector steps per loop $S$ & 4 & 4 & 4 & 4 \\
    \bottomrule
    \end{tabular}
    \label{tab:llada_hyperparams}
\end{table*}

\textbf{Baselines}~
For all prior work results, other than ReMeDi \citep{huang2025don}, we use open-source weights and evaluate using with batch size 1, maximum generation length of 1024, and early stopping on the \texttt{[EOS]} token.
For ReMeDi Instruct, we report values from \citet{huang2025don}, as we were unable to reproduce or improve upon their results using the open-source model provided.

All model evaluations, other than for the off-the-shelf LLaDA-Base model,  were performed using the Nemo Skills library, using default chat templates from this library and the prompt templates described above.
For LLaDA-Base, since the Nemo Skills repository injects chat templates by default, we evaluate using the lm-eval-harness repository.

For ReMDM, we follow the algorithm proposed by \citet{wang2025remasking} for applying this method to LLaDA.
Namely, for each block of 32 tokes, once 28 tokens have been generated, we enable a ReMDM loop, where for 32 iterations we remask 2 tokens that had the lowest confidence at the time at which they were decoded and unmask 2 tokens based on their confidence.
Hence, at the end of the ReMDM loop, there are still 28 unmasked and 4 masked tokens, at which point we finish generating using the standard LLaDA confidence-based sampling.
See Appendix \ref{appsubsec:remdm_llada} for more a extensive hyperparameter sweep for using ReMDM.

\subsection{Guided Molecule Design}\label{appsubsec:exp_details_guidance}
For this experiment, we follow the setup detailed in \citet{schiff2024simple}.

\textbf{Dataset}~
We train on the QM9 dataset \citep{ramakrishnan2014quantum, ruddigkeit2012enumeration}, which consists of $\sim$133k molecules represented as SMILES strings \citep{weininger1988smiles}.
We use the RDKit library \citep{landrum2013rdkit} to add the ring count and drug likeness (QED; \cite{bickerton2012quantifying}) annotations.
The dataset was tokenized using a regular expression tokenizer \citep{schwaller2019molecular}.
We use sequence length of $L=32$, with right-padding.

For each property, we generate binary labels that indicate whether a sample is below or above the 90th percentile of training samples.
For discrete classifier-free-guidance \citep{schiff2024simple}, we train with this label for conditional models, and randomly `drop it out' 10 percent of the time by replacing it with a `masked' label to simulate unconditional modeling.

\textbf{Hyperparameters}~
Hyperparameters follow \citet{schiff2024simple}.
Namely, we use a DiT-style \citep{peebles2023scalable} backbone with 92.4M parameters.
Models were trained with a batch size 2048, and we perform 25k gradient updates.
We use a maximum learning rate of 3$\mathrm{e}^{-4}$ that we linearly warm-up to for 1000 steps.
After this peak we apply cosine decay until a minimum learning rate of 3$\mathrm{e}^{-6}$.
We use the \textsc{ADAM-w} with beta parameters $(0.9, 0.999)$.

Of note, when training \method~models for ring count, we found it beneficial to eliminate the `copy over' parameterization of the denoising network $\x_\theta$ proposed in \citet{sahoo2024simple}.
That is, we do not enforce that $\x_\theta$ simply copy over any token positions $\ztl \neq \mask$.
For training models for the QED property, we maintained this copy-over parameterization.

\textbf{Compute Resources}~
We trained our model on 8 A5000 GPUs.

\textbf{Evaluation}~
We generate 1024 samples from our model using various unmasking budgets $T$ and guidance temperature $\gamma$.
Of note, when applying the corrector model forward passes, we only use the conditional model, i.e., $\gamma =1$.

We use the RDKit library to parse generated samples.
Of the valid strings (those that  can be parsed) we retain unique samples that are not found in the original QM9 dataset (novel).
We then use RDKit to measure the property of interest for these novel samples.

\textbf{Sampling Hyperparameters}~
For both ring count and QED maximization we use $\omega = T / 2$ for corrector loop frequency and $S = T /16$ for steps per loop.
Evaluation was performed with an exponential moving average (EMA) checkpoint.
During training we used an EMA decay factor of $0.9999$.

\textbf{Baselines}~
Values for AR, UDLM, MDLM, and ReMDM were taken from \citet{wang2025remasking}.

For ReMeDi \citep{huang2025don} and GStar \citep{meshchaninov2025guided}, we re-implement these methods and train them on the QM9 dataset.
Both of these methods entail training a classifier that predicts incorrect tokens to remask.

For ReMeDi, we jointly train their ``Unmasking Policy Stream'' (UPS) alongside an MDLM model.
We use the ``dual stream'' architecture recommended by \citet{huang2025don}, with 4 DiT layers for the UPS, connected to the main backbone at layers 0, 3, 7, and 10.
We set the UPS loss weight at $\lambda_{\text{UPS}} = 0.3$.
For the proportion of unmasked tokens that get flipped to random tokens we use $\rho_{t, \text{incorrect}} = 0.1$.
Finally, as ReMeDi was not originally proposed within th context of classifier-free guidance (CFG), to have a fair comparison, we perform a hyperparameter sweep for this model over the following parameters: whether to apply CFG to only the ``Token Prediction Stream'' (i.e., the MDLM output) or also to the UPS and whether to use sampling or greedy-max selection when choosing which tokens to remask.
We found best results when applying CFG to both streams and when using proportional sampling, not greedy-max selection, for remasking positions.

In GStar, the error prediction head is trained with a frozen diffusion backbone model, hence we use the pre-trained MDLM from \citet{wang2025remasking} and freeze its weights.
The error prediction head consists of a \texttt{LayerNorm} and linear projection from the final hidden token representations of the frozen backbone.
We employ the ``GStar+'' sampling algorithm and for fair comparison we perform a sweep over the following parameters: whether to apply the CFG to only the MDLM output or also to the remasking prediction head, whether to use Gumbel top-$k$ or greedy top-$k$ selection during remasking, the temperature for both the MDLM (sweep over $\{0.7, 1.0, 1.3\}$) and remasking head output (sweep over $\{2, 4, 8, 16\}$), and the remasking switch-on time (sweep over $\{0.1, 0.2, 0.3, 0.4\}$).
For the ring count property, we found best results using Gumbel top-$k$, switch time 0.2, denoiser temperature 0.7, and prediction head temperature 16.
For the QED property, we found best results using greedy top-$k$, switch time 0.4, and denoiser temperature 1.0.

\subsection{Unconditional Text Generation}\label{appsubsec:exp_details_text}
For this experiment, we follow the setup described in \citet{sahoo2024simple}.

\textbf{Dataset}~
We train models on the OpenWebText (OWT; \cite{Gokaslan2019OpenWeb}) dataset.
We tokenized using the \texttt{gpt-2} \citep{radford2019language} tokenizer and created sequences of $L=1024$ tokens by wrapping samples and separating them with an \texttt{[EOS]} token. 
We also also place an \texttt{[EOS]} token at the beginning and end of each sequence.

\textbf{Hyperparameters}~
As in \citet{sahoo2024simple}, we use a DiT backbone with 170M parameters.
We used a batch size of 512 and applied 1M gradient updates.
We use a constant learning rate of 3$\mathrm{e}^{-4}$ that we linearly warm-up to for 2500 steps.
We use the \textsc{ADAM-w} optimizer with beta parameters (0.9, 0.999).
As described in Appendix \ref{appsubsec:exp_details_benchmarks}, we use a \texttt{min\_t} value of 1$\mathrm{e}^{-1}$, when training \method~models on OWT.

For training \method~in this setting, we found it beneficial to eliminate the `copy over' parameterization of the denoising network $\x_\theta$.

\textbf{Compute Resources}~
We trained our model on 32 H100 GPUs.

\textbf{Evaluation}~
We follow the evaluation protocol from \citet{wang2025remasking}.
Specifically, we generate 5000 samples and compute the MAUVE metric \citep{pillutla2021mauve}, generative perplexity under the \texttt{gpt2-large} model, and entropy of generated tokens.

\textbf{Modified Training Objective}~
For this experiment, we use a slightly modified objective relative to that presented in Algorithm \ref{alg:train}.
Specifically, we adjust the self-correction term $\mathcal{L}^{\mathpzc{SC}}$ so that gradients for tokens which the original denoiser model incorrectly predicts are steeper.
To accomplish this, we replace the self-correction loss term from Algorithm \ref{alg:train} with the following:
\begin{align}\label{eq:modified_sc}
    \Tilde{\mathcal{L}}^{\mathpzc{SC}}(\theta) := \frac{\atdot}{1-\at}\sum_{\ell=1}^L \delta_{\ytl, \xl} \langle \xl_\theta(\ytL), \xl \rangle + (1 - \delta_{\ytl, \xl}) \log\langle \xl_\theta(\ytL), \xl \rangle.
\end{align}
The modified loss term in (\ref{eq:modified_sc}) has the desirable property that token positions where the original denoiser model predicts a mistake, i.e. $\delta_{\ytl, \xl} = 0$, are prioritized.
This is accomplished by the fact that the gradient for the standard cross-entropy loss used for these token positions is steeper than that for positions where the denoising model is `already correct', i.e., $\delta_{\ytl, \xl} = 1$, where the loss is taken without the logarithm.

\textbf{Modified Sampling Procedure}~
Additionally, for the samples presented in Table \ref{tab:owt}, we use a slightly modified sampling procedure.
Rather than simply using $\argmax$ decoding at every token position for the $\texttt{sample}$ sub-routine in the \texttt{corrector} function in Algorithm \ref{alg:sample}, at each corrector step, we sort all positions by the confidence of the $\argmax$ token.
Then, for the top-$k$ positions, we use the $\argmax$ token, and for the remaining positions, we leave them unchanged.
We use $k = 100$ and found that this procedure better encouraged sample diversity and improved MAUVE scores.
In Algorithm \ref{alg:corrector_sample_owt}, we detail the top-$k$ sampling sub-routine that we use for corrector sampling in the OWT experiments.

\begin{algorithm}[!ht]
    \small
    \caption{Corrector \texttt{sample} sub-routine used in OWT experiments}
    \label{alg:corrector_sample_owt}
    \begin{algorithmic}[1]
    \STATE \textbf{Input:} Corrector model input $\ytL$, corrector model output $\texttt{corrector\_logits}$, parameter $k$    
    \STATE $\hat{\y}^{1:L} \leftarrow \argmax(\texttt{corrector\_logits})$
    \STATEX \textcolor{ourgray}{\texttt{\# `gather' is an API that returns values of its first input at the indices specified by its second input, see for example torch.gather (with dim=-1)}}
    \STATE $\texttt{confidence\_scores} \leftarrow \texttt{gather}(\texttt{corrector\_logits}, \hat{\y}^{1:L})$
    \FOR{$\ell=1$ {\bfseries to} $L$}
        \IF{$\hat{\y}^\ell == \ytl$}
            \STATEX \textcolor{ourgray}{\texttt{\# Only consider positions where corrector output is different from input}}
            \STATE $\texttt{confidence\_scores}^\ell \leftarrow -\infty$
        \ENDIF
    \ENDFOR
    \STATEX \textcolor{ourgray}{\texttt{\# `top\_k' is an API that returns the top k argument positions for its first input}}
    \STATE $\texttt{top\_k\_indices} \leftarrow \texttt{top\_k}(\texttt{confidence\_scores}, k)$
    \FOR{$\ell=1$ {\bfseries to} $L$}
        \IF{$\ell \in \texttt{top\_k\_indices}$}
            \STATE $\ytl \leftarrow \hat{\y}^\ell$
        \ELSE
            \STATE $\ytl \leftarrow \ytl$
        \ENDIF
    \ENDFOR
    \STATE \textbf{Return} $\ytL$
    \end{algorithmic}
\end{algorithm}

\textbf{Sampling Hyperparameters}~
For \method, we match the inference budget of baseline results by using number of unmasking steps equal to $T / 4$ per column, performing a corrector loop at every iteration, $\omega = 1$, and applying 3 corrector steps per loop, $S = 3.$
We use $k = 100$ as the top-$k$ parameter for Algorithm \ref{alg:corrector_sample_owt}.
Evaluation was performed with an EMA checkpoint.
During training we used an EMA decay factor of $0.9999$.

\textbf{Baselines}~
Values for the baseline models were taken from \citet{wang2025remasking}, except for PRISM results which were taken from \citet{kim2025fine}.
For PRISM, results correspond to the `PRISM-loop' method presented in Table 3 of \citet{kim2025fine}.

For GIDD \citep{von2025generalized}, we use the $p_u = 0.2$ open-source checkpoint provided by the authors.
Since this model was trained on sequence length of $L=512$, in contrast to the other comparisons and our model which use $L=1024$, we evaluate the GIDD model with $L=512$.
We use temperature parameter equal to 0.1, and we allow an equal number of unmasking and self-correction steps.
Since the self-correction algorithm proposed in GIDD can terminate prior to exhausting the full budget, in Table \ref{tab:gidd_params} we denote the GIDD configurations used to match the GIDD results to the various $T$ budgets in Table \ref{tab:owt}.
Note that for certain $T$, we err on the side of allowing the GIDD model to exceed the total budget $T$.

\begin{table}[!ht]
    \centering
    \small
    \caption{Parameters used for GIDD \citep{von2025generalized} unmasking and self-correction budgets corresponding to the results in Table \ref{tab:owt}.}
    \begin{tabular}{lccc}
    \toprule
    $T=$ & Unmasking steps & Maximum Self Correction Steps & Avg. NFEs \\
    \midrule
    128 & 64 & 64 & 123.9 \\
    256 & 128 & 128 & 212.6 \\
    512 & 512 & 512 & 601.3 \\
    1024 & 1024 & 1024 & 1114.2 \\
    \bottomrule
    \end{tabular}
    \label{tab:gidd_params}
\end{table}

\section{Additional Experimental Results}\label{appsec:exp_results}

\subsection{Additional ReMDM Configurations for LLaDA}
\label{appsubsec:remdm_llada}

The ReMDM results in Table \ref{tab:performance_results}
use the configuration recommended for LLaDA in
\citet{wang2025remasking}.
To examine whether alternative configurations improve
performance in our SFT setting, we conduct an additional
hyperparameter sweep using the same vanilla SFT checkpoint.
We find that random remasking, rather than confidence-based
remasking, together with multiple ReMDM loops during each
block's generation, rather than a single loop near its end,
improves benchmark performance.

Table \ref{tab:remdm_llada_additional} compares the two
ReMDM configurations.
The additional sweep improves performance on all four
benchmarks relative to the paper-recommended configuration,
with the largest gains on HumanEval and MBPP.
Relative to vanilla SFT without correction, however,
the updated configuration improves HumanEval while
remaining slightly below the baseline on GSM8K and MATH
and below it on MBPP.
\method~with Max Sampling continues to outperform both
ReMDM configurations across all four benchmarks.

\begin{table}[t]
    \centering
    \small
    \caption{Additional ReMDM configurations applied to our
    vanilla LLaDA-Base 8B SFT checkpoint.
    The paper-recommended configuration corresponds to the
    ReMDM row under our SFT baselines in
    Table \ref{tab:performance_results}.
    The additional sweep uses random remasking and multiple
    ReMDM loops per block.
    Values are Pass@1; best values are \textbf{bolded}.}
    \label{tab:remdm_llada_additional}
    \begin{tabular}{lcccc}
        \toprule
        Model / configuration
        & HumanEval & MBPP & GSM8K & MATH \\
        \midrule
        Vanilla SFT
        & 58.54 & 56.88 & 88.86 & 46.60 \\
        + ReMDM (paper-recommended)
        & 56.10 & 50.00 & 88.48 & 46.22 \\
        + ReMDM (additional sweep)
        & 62.20 & 54.50 & 88.70 & 46.50 \\
        \midrule
        \method~SFT + Max Sampling
        & \textbf{72.56} & \textbf{69.31}
        & \textbf{92.19} & \textbf{55.06} \\
        \bottomrule
    \end{tabular}
\end{table}

\subsection{Comparison with Extended SFT Training}
\label{app:extended_training}

The additional self-correction pass increases the per-step
training cost of \method.
To assess whether additional optimization of the standard
MDM objective can recover its performance gains, we compare
against a vanilla SFT baseline trained for twice the original
number of optimization steps.

\paragraph{Results}
Table \ref{tab:llada_extended_training} shows that \method~outperforms the extended-training baseline on all four benchmarks implying that doubling the number of standard MDM training steps does not close the performance gap to the \method~framework.

\begin{table}[t]
    \centering
    \small
    \caption{Pass@1 comparison with extended vanilla SFT.
    The $2\times$ baseline receives twice the optimization steps
    of the original vanilla SFT baseline.
    Best values are \textbf{bolded}.}
    \label{tab:llada_extended_training}
    \begin{tabular}{lcccc}
        \toprule
        Model & HumanEval & MBPP & GSM8K & MATH \\
        \midrule
        Vanilla SFT
        & 58.54 & 56.88 & 88.86 & 46.60 \\
        Vanilla SFT ($2\times$ steps)
        & 62.80 & 56.35 & 89.61 & 48.08 \\
        \midrule
        \method~SFT + Max Sampling
        & \textbf{72.56} & \textbf{69.31} & \textbf{92.19} & \textbf{55.06} \\
        \bottomrule
    \end{tabular}
\end{table}

\subsection{Quality-Efficiency Tradeoffs}\label{appsubsec:exp_results_tps}
In Table \ref{tab:quality_efficiency}, we compare the baseline method to the \method~Max, Balanced, and Fast configurations and report throughput (in tokens per second, TPS) as well as the accuracy and efficiency gains.
In Figure \ref{fig:pareto_tps}, we also repeat our Pareto frontier results from Figure \ref{fig:llada_results_pareto}, with throughput reported.

\begin{table*}[ht!]
    \centering
    \small
    \caption{Quality vs. Efficiency Metrics. We report Accuracy (\%), Average Number of Function Evaluations (NFEs), and Throughput (Tokens/Second). Best accuracy gain and efficiency results are \textbf{bolded}.
    Improvements over the baseline are highlighted in \textcolor{OliveGreen}{green}.}
    \begin{tabular}{llcccc}
    \toprule
    & & \textbf{HumanEval} & \textbf{MBPP} & \textbf{GSM8K} & \textbf{MATH} \\
    \midrule
    \rowcolor{lightgray!25}
    \multicolumn{6}{l}{\textit{Baseline (Vanilla SFT)}} \\
    & Accuracy & 58.54 & 56.88 & 88.86 & 46.60 \\
    & Avg. NFEs & 247.8 & 290.3 & 214.2 & 259.4 \\
    & Tokens/Second & 123.14 & 131.04 & 148.92 & 153.15 \\
    \midrule
    \rowcolor{lightgray!25}
    \multicolumn{6}{l}{\textit{\method~Max}} \\
    & Accuracy (+ gain) & \textbf{72.56} (\textcolor{OliveGreen}{+14.0}) & \textbf{69.31} (\textcolor{OliveGreen}{+12.4}) & \textbf{92.19} (\textcolor{OliveGreen}{+3.3}) & \textbf{55.06} (\textcolor{OliveGreen}{+8.5}) \\
    & Avg. NFEs ($\times$ speedup) & 499.2 (0.50$\times$) & 818.3 (0.35$\times$) & 342.5 (0.63$\times$) & 747.5 (0.35$\times$) \\
    & Tokens/Second ($\times$ speedup) & 59.66 (0.48$\times$) & 37.48 (0.29$\times$) & 103.83 (0.70$\times$) & 61.38 (0.40$\times$) \\
    \rowcolor{lightgray!25}
    \multicolumn{6}{l}{\textit{\method~Balanced}} \\
    & Accuracy (+ gain) & 70.12 (\textcolor{OliveGreen}{+11.6}) & 66.93 (\textcolor{OliveGreen}{+10.1}) & 91.28 (\textcolor{OliveGreen}{+2.4}) & 54.36 (\textcolor{OliveGreen}{+7.8}) \\
    & Avg. NFEs ($\times$ speedup) & 213.9 (\textcolor{OliveGreen}{1.16$\times$}) & 218.7 (\textcolor{OliveGreen}{1.33$\times$}) & 228.7 (0.94$\times$) & 321.2 (0.81$\times$) \\
    & Tokens/Second ($\times$ speedup) & 137.07 (\textcolor{OliveGreen}{1.11$\times$}) & 131.15 (1.00$\times$) & 158.33 (\textcolor{OliveGreen}{1.06$\times$}) & 141.94 (0.93$\times$) \\
    \rowcolor{lightgray!25}
    \multicolumn{6}{l}{\textit{\method~Fast}} \\
    & Accuracy (+ gain) & 60.98 (\textcolor{OliveGreen}{+2.4}) & 62.70 (\textcolor{OliveGreen}{+5.8}) & 90.90 (\textcolor{OliveGreen}{+2.0}) & 49.56 (\textcolor{OliveGreen}{+3.0}) \\
    & Avg. NFEs ($\times$ speedup) & \textbf{61.4} (\textcolor{OliveGreen}{4.04$\times$}) & \textbf{110.9} (\textcolor{OliveGreen}{2.62$\times$}) & \textbf{119.5} (\textcolor{OliveGreen}{1.79$\times$}) & \textbf{81.2} (\textcolor{OliveGreen}{3.19$\times$}) \\
    & Tokens/Second ($\times$ speedup) & \textbf{455.01} (\textcolor{OliveGreen}{3.70$\times$}) & \textbf{251.92} (\textcolor{OliveGreen}{1.92$\times$}) & \textbf{304.33} (\textcolor{OliveGreen}{2.04$\times$}) & \textbf{512.27} (\textcolor{OliveGreen}{3.34$\times$}) \\
    \bottomrule
    \end{tabular}
    \label{tab:quality_efficiency}
\end{table*}

\begin{figure}
    \centering
    \includegraphics[width=\linewidth]{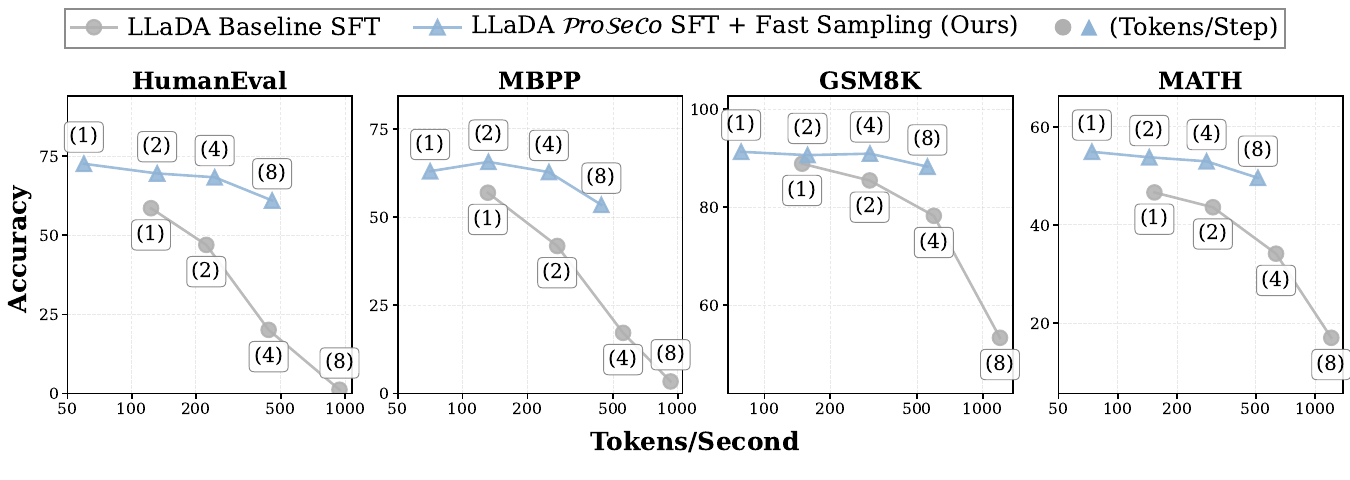}
    \caption{Pareto frontier analysis with throughput (tokens per second, TPS) reported.
    }
    \label{fig:pareto_tps}
\end{figure}

\begin{figure}
\centering
\includegraphics[width=\linewidth]{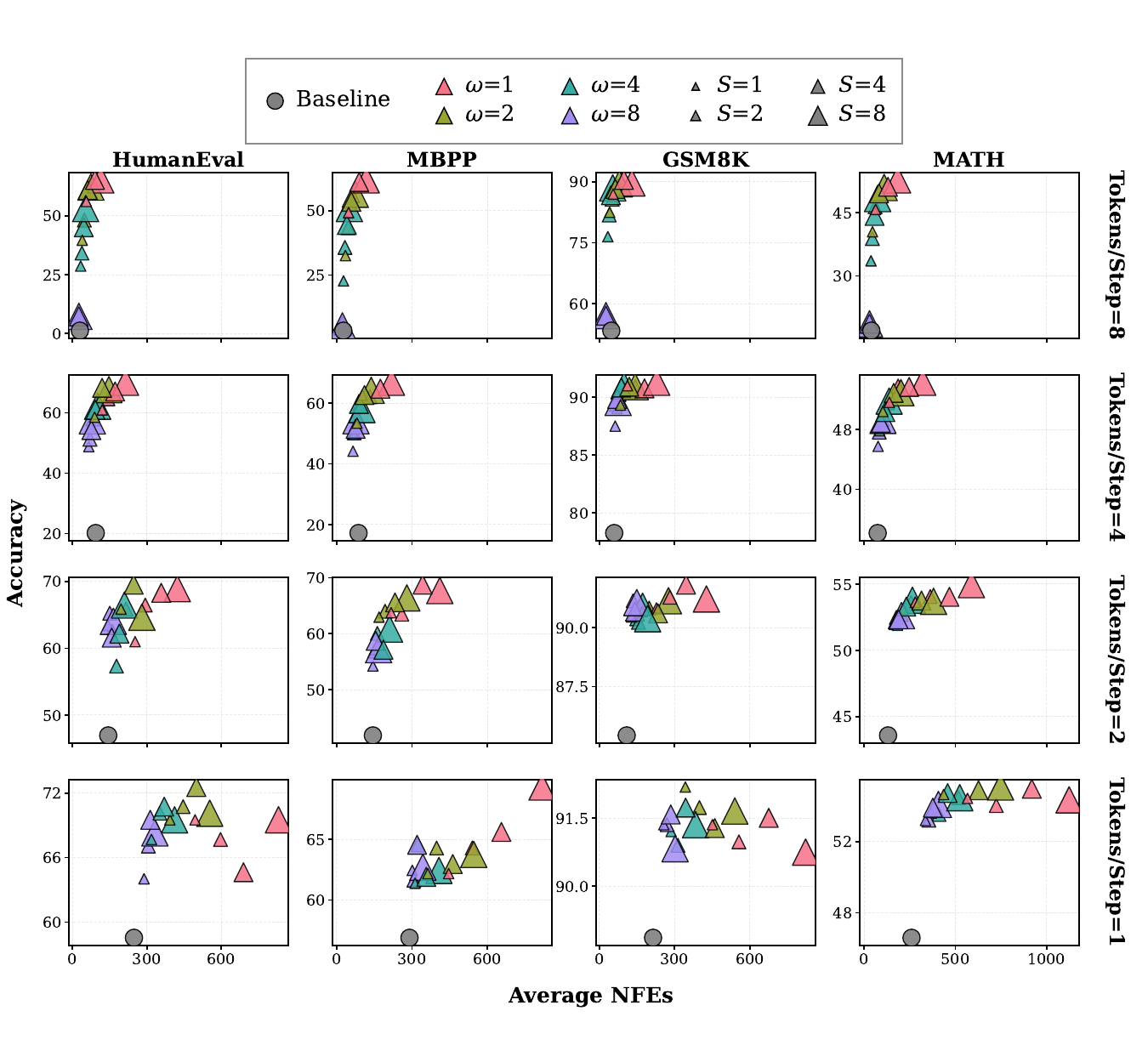}
\caption{Ablation: Performance across various configurations of corrector steps.
Frequency $\omega \in \{1, 2, 4, 8\}$ (denoted by color) and $S \in \{1, 2, 4, 8\}$ (denoted by marker size).
}
\label{fig:llada_results_budget}
\end{figure}

\subsection{Ablation: Robustness to Corrector Parameters}\label{appsubsec:ablation}
In Figure \ref{fig:llada_results_budget}, we demonstrate that our model, which underwent SFT with \method~training, is highly robust to the choice of the corrector sampling parameters, i.e., the frequency of corrector loops $\omega$ and number of corrector steps per loop $S$.
We look at various levels of unmasking parallelism, ranging from fully sequential at $\text{tokens/step} = 1$ to $\text{tokens/step} = 8$, and evaluate our model with $\omega \in \{1, 2, 4, 8\}$ and $S \in \{1, 2, 4, 8\}$.
In particular, we see that at every level of decoding parallelism, our model can beat the best performing Baseline SFT results for all choices of $\omega$ and $S$.

A few additional observations from Figure \ref{fig:llada_results_budget}: in the first two rows, corresponding to `fast sampling' regimes, in order to match or exceed the best baseline accuracy, which is attained with $\text{tokens/step} = 1$, we require more frequent corrector loops $\omega \in \{1, 2\}$.
In these settings we can drastically improve efficiency relative to the baseline model.
In this fast sampling regime, we also see a general trend that, for a fixed corrector budget (i.e., $S \cdot L / \omega$) , more frequent but shorter correction loops are typically more effective.
Additionally, when unmasking sampling increases to $\text{tokens/ step} \in \{1, 2\}$, we find the expected trend that scaling both frequency of correction loops and number of steps per loop leads almost uniformly leads to improved sample quality, at the cost of additional NFEs.

\subsection{Selecting the Max, Balanced, and Fast Points}
\label{appsubsec:correction_configs}
From the hyperparameter sweep described in Appendix \ref{appsubsec:ablation}, we selected representative points as \method~Max, Balanced, and Fast.
The chosen parameters for each of these are detailed in Table \ref{tab:llada_hyperparams}.
For the Max configuration, we selected the best overall \method~performance from the sweep performed in Figure \ref{fig:llada_results_budget}.
For the Fast configuration, we selected points from parallel generation of either $\text{tokens/step} = 4$ or $\text{tokens/step} = 8$ where we attain better accuracy than the best baseline performance, which uses $\text{tokens/step} = 1$.
For the Balanced configuration, we anchor $\text{tokens/step} = 4$, and select a configuration that pushes the boundary to the north-west across all benchmarks.

\subsection{Ablation: Self-correction Sampling with Standard MDM}\label{appsubsec:ablation_correct_mdm}
In Table \ref{tab:ablation_correct_mdm}, we demonstrate the importance of \method~training, by comparing self-correction sampling applied to a model with standard MDM SFT vs. one with \method~loss and Algorithm \ref{alg:train}.
These results underscore the importance of training a model to self-correct.
With standard MDM training, the model does not learn to produce meaningful predictions at already unmasked token locations.
Our method unlocks this ability, and therefore benefits from the self-correction sampling algorithm.

\begin{table}[!ht]
    \centering
    \small
    \caption{Ablation: Applying self-correction sampling to models trained with standard MDM loss leads to worse performance, underscoring the importance of the \method~framework.
    Pass@1 accuracy (\%) performance is reported.}
    \label{tab:ablation_correct_mdm}
    \begin{tabular}{l cccc}
        \toprule
        \textbf{Training + Sampling} & \textbf{HumanEval} & \textbf{MBPP} & \textbf{GSM8K} & \textbf{MATH} \\
        \midrule
        Vanilla SFT + Standard MDM & 58.54 & 56.88 & 88.86 & 46.60 \\
        Vanilla SFT + \method~Balanced & 1.83  & 1.85  & 60.42 & 27.06    \\
        \midrule
        \method~SFT + \method~Balanced & 72.56 & 69.31 & 92.19 & 55.06 \\
        \bottomrule
    \end{tabular}
\end{table}

\subsection{Ablation: Selecting \texorpdfstring{$\lambda_t$}{lambda\_t}}\label{appsubsec:ablation_lambda}
Below we present the ablation on the self-correction loss term weight.
We run training for discrete classifier guidance on the ring count property with fixed $\lambda \in \{0.1, 1, 10\}$ and a time-varying weighted $ \lambda_t \in \{ 0.1\cdot \tfrac{\atdot}{(1-\at)}, 1\cdot \tfrac{\atdot}{(1-\at)}, 10\cdot \tfrac{\atdot}{(1-\at)}\}.$
We report results in Figure \ref{fig:lambda_ablation}.
We find that using the $\tfrac{\atdot}{1-\at}$ weighting consistently improves results.
Furthermore, our model is robust to different scalings of this time-dependent weighting.

\begin{figure}
    \centering
    \includegraphics[width=0.55\linewidth]{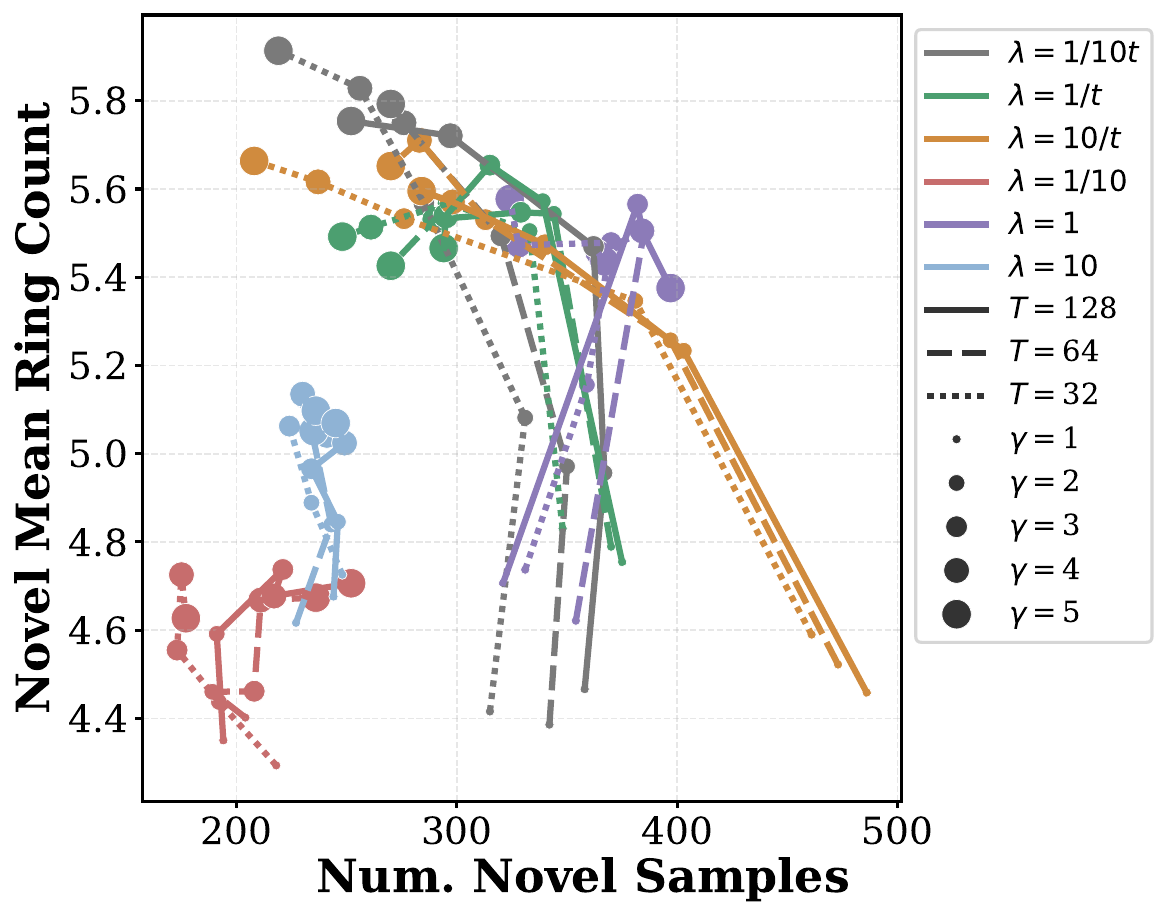}
    \caption{Ablation: Selecting self-correction loss weight $\lambda_t$.
    We retrain the discrete classifier model for maximizing ring count with various $\lambda_t$ values.
    }
    \label{fig:lambda_ablation}
\end{figure}

\subsection{Error Bars for Table \ref{tab:owt}}
In Table \ref{tab:owt_error_rows}, we report mean $\pm$ standard error from four random seeds used during sampling for our model's results on unconditional generation (Table \ref{tab:owt}).

\begin{table*}[t]
    \centering
    \small
    \caption{Mean $\pm$ standard error for OWT unconditional generation metrics for \method.}
    \begin{tabular}{l|ccc}
    \toprule
    $T=$ & MAUVE ($\uparrow$) & Gen. PPL ($\downarrow$) & Entropy ($\uparrow$) \\
    \midrule
    128  & 0.311 $\pm$ 0.006 & 23.045 $\pm$ 0.031 & 5.452 $\pm$ 0.002 \\
    256  & 0.572 $\pm$ 0.006 & 16.52 $\pm$ 0.014  & 5.376 $\pm$ 0.005 \\
    512  & 0.622 $\pm$ 0.005 & 13.196 $\pm$ 0.031 & 5.293 $\pm$ 0.001 \\
    1024 & 0.592 $\pm$ 0.006 & 10.92 $\pm$ 0.014  & 5.221 $\pm$ 0.005 \\
    \bottomrule
    \end{tabular}
    \label{tab:owt_error_rows}
\end{table*}

\section{Generated Samples}\label{appsec:samples}
\subsection{LLaDA \texorpdfstring{\method}{ProSeCo}~SFT Samples}\label{appsubsec:samples_llada}
In Figures \ref{fig:humaneval_max_sample} and \ref{fig:gsm8k_max_sample1}, we present sample generations for the HumanEval and GSM8K datasets, respectively, using the maximum accuracy configuration for each benchmark (see Table \ref{tab:llada_hyperparams}).

\subsection{\texorpdfstring{\method}{ProSeCo}~Unconditional Generation Samples}\label{appsubsec:samples_owt}
In Figure \ref{fig:owt_256_samples}, we present a sample generated from the \method~model trained on OWT.
We use total sample budget of $T=256$, which consists of $64$ unmasking steps, a corrector loop every $\omega=1$ step, and $S=3$ corrector steps per loop.

\section{Assets}\label{appsec:assets}
In Table \ref{tab:datasets}, we list the corresponding licenses for datasets used in this work.
\begin{table}[ht!]
    \centering
    \small
    \caption{Datasets and corresponding licenses.}
    \begin{tabular}{lc}
    \toprule
        Dataset & Licence \\
        \midrule
         GSM8K \citep{cobbe2021training} & MIT \\
         HumanEval \citep{chen2021evaluating} & MIT \\
         Llama-Nemotron: Efficient Reasoning Models \citep{bercovich2026extending} & CC BY 4.0 \\
         MBPP \citep{austin2021program} & MIT \\
         MinveraMath \citep{hendrycks2021measuring} & MIT \\
         OpenWebText \citep{Gokaslan2019OpenWeb} & Creative Commons CC0 license (``no rights reserved'')\\
        QM9 \cite{ramakrishnan2014quantum, ruddigkeit2012enumeration}  & N/A \\
        \bottomrule
    \end{tabular}
    \label{tab:datasets}
\end{table}

In Table \ref{tab:software}, we list the corresponding licenses for software packages used in this work.

\begin{table}[ht]
    \centering
    \small
    \caption{Software and corresponding licenses.}
    \begin{tabular}{ll}
    \toprule
        Library & License \\
        \midrule
        HuggingFace~\citep{wolf2019huggingface} & Apache 2.0 \\
        Hydra~\citep{Yadan2019Hydra} & MIT \\
        Fast-DLLM~\citep{wu2025fast} & Apache 2.0 \\
        Language Model Evaluation Harness \citep{eval-harness} & MIT \\
        Matplotlib~\citep{Hunter:2007} & \href{https://matplotlib.org/stable/users/project/license.html}{Matplotib license} \\
        Mauve \citep{pillutla2021mauve} & GNU General Public License, Version 3 \\
        MDLM \citep{sahoo2024simple} & Apache 2.0 \\
        Nemo Skills & Apache 2.0 \\
        NumPy~\citep{harris2020array} & \href{https://numpy.org/doc/stable/license.html}{NumPy license} \\
        OmegaConf & BSD 3-Clause \\
        Pandas \citep{reback2020pandas} & BSD 3-Clause ``New" or ``Revised" \\        PyTorch~\citep{Paszke_PyTorch_An_Imperative_2019} & BSD-3 Clause \\
        PyTorch Lightning~\citep{Falcon_PyTorch_Lightning_2019} & Apache 2.0 \\
        RDKit \cite{landrum2013rdkit} & BSD 3-Clause ``New" or ``Revised" \\
        Seaborn~\citep{Waskom2021} & BSD 3-Clause ``New" or ``Revised" \\
        TorchMetrics & Apache 2.0 \\
        UDLM \citep{schiff2024simple} & Apache 2.0 \\
        \bottomrule
        \end{tabular}
    \label{tab:software}
\end{table}

\section{Impact Statement}\label{appsec:impacts}
Our paper is intended as a contribution to the machine learning community.
Our work holds the promise for positive impact by accelerating inference and improving quality of language models.
While, our method is also subject to the dangers of misuse present in language modeling, especially in the application to biological sequences, given the limited scale of models that we explore, the experiments described in this work do not pose any significant risks.

\section{LLM Usage Disclosure}\label{appsec:llm}
Large Language Models (LLMs) were used to format and edit this manuscript and for formatting figures.
Additionally, integrated LLMs in the code editing software used while developing this work were used for auto-complete suggestions and in debugging.

LLMs were \textbf{not} used for the core development of our methodology, derivation of key results, nor the programming of our main code and experiments.

\newpage
\begin{figure}[ht!]
\centering
\small
\begin{samplebox}{HumanEval Sample}
\textbf{Prompt shown to the model}
\begin{SampleVerbatim}
Here is a problem for which you need to generate/complete code:
  
def make_a_pile(n):
	"""
	Given a positive integer n, you have to make a pile of n levels of stones.
	The first level has n stones.
	The number of stones in the next level is:
		- the next odd number if n is odd.
		- the next even number if n is even.
	Return the number of stones in each level in a list, where element at index
	i represents the number of stones in the level (i+1).

	Examples:
	>>> make_a_pile(3)
	[3, 5, 7]
	"""

  Please continue to complete the function with python programming language. You are not allowed to modify the given code and do the completion only.

  The solution should be in the following format:
  ```python
  # Your code here
  ```
\end{SampleVerbatim}

\textbf{Model output}
\begin{SampleVerbatim}
```python
def make_a_pile(n):
    """
    Given a positive integer n, you have to make a pile of n levels of stones.
    The first level has n stones.
    The number of stones in the next level is:
        - the next odd number if n is odd.
        - the next even number if n is even.
    Return the number of stones in each level in a list, where element at index
    i represents the number of stones in the level (i+1).

    Examples:
    >>> make_a_pile(3)
    [3, 5, 7]
    """
    pile = []
    for i in range(1, n + 1):
        if i % 2 == 1:          # odd level -> next odd number
            pile.append(2 * i - 1)
        else:                   # even level -> next even number
            pile.append(2 * i)
    return pile
```
\end{SampleVerbatim}
\end{samplebox}
\caption{Example of \method's generated response on the HumanEval dataset using the max accuracy configuration (see Table \ref{tab:llada_hyperparams}).
}
\label{fig:humaneval_max_sample}
\end{figure}
\newpage
\begin{figure}
\centering
\small
\begin{samplebox}{GSM8K Sample; Expected answer: 6}
\textbf{Prompt shown to the model}
\begin{SampleVerbatim}
Solve the following math problem. Show your step-by-step reasoning, then provide the final answer inside \boxed{}.

  Solo has to read 4 pages from his Science textbook, 20 pages from his Social Studies textbook, 7 pages from his History textbook and 8 pages from his Geography textbook. Solo read 15 pages on Monday. If he has 4  more days to complete all his reading, how many pages does he need to read, on average, in one day?

\end{SampleVerbatim}

\textbf{Model output}
\begin{SampleVerbatim}
**Step‑by‑step reasoning**

1. **Total pages to be read**  
   \[
   4\ (\text{Science}) + 20\ (\text{Social Studies}) + 7\ (\text{History}) + 8\ (\text{Geography}) = 39\ \text{pages}
   \]

2. **Pages already read**  
   \[
   15\ \text{pages on Monday}
   \]

3. **Pages left to read**  
   \[
   39 - 15 = 24\ \text{pages}
   \]

4. **Number of days left**  
   \[
   4\ \text{more days}
   \]

5. **Average pages per day**  
   \[
   \frac{24\ \text{pages}}{4\ \text{days}} = 6\ \text{pages per day}
   \]

\[
\boxed{6}
\]
\end{SampleVerbatim}
\end{samplebox}

\caption{Example of \method's generated response on the GSM8K dataset using the max accuracy configuration (see Table \ref{tab:llada_hyperparams}).}
\label{fig:gsm8k_max_sample1}
\end{figure}
\newpage
\begin{figure}[ht!]
\centering
\scriptsize
\begin{samplebox}{OWT Generated Sample}
``\textbf{\textbar endoftext\textbar} campaigns trying to figure out how to raise money for a ballot measure America’s political dailies have not adopted an information-delivery strategy. Teich’s most recent foray was, according to its spokeswoman, a brand new social media campaign that it took a break from advertising the previous year, and launched on its website. Both stem from a tradition of public-sector philanthropy, and position the company as a public utility whose funding is in no way tied to government spending or health care. We have reached out to Teich for more information, and so far have received no response.\newline\newline But Teich has all but abandoned its information delivery strategy—or at least when it comes to what it calls its “frontiers.''\newline\newline While its Indiegogo campaign has been focused on social issues, much of it has centered on public health and safety initiatives in an effort to reduce the number of traffic deaths, make driving less dangerous, reduce dependence on oil and to improve the health of people in a developing country. Its\newline\newline budgeted nearly \$60 million in fiscal year 2016, according to Teich Institute, a nonprofit that tracks Teich’s spending.\newline\newline Back in 2014, the year it launched, the company calculated how much it could spend on new software, or new nurses, or programs that focus on technology. It spends a big chunk of that money on budgeting—but there are different approaches to health care spending. For instance, some decide what kind of money they spend; others decide how much money they spend on a particular service. So how much Teich spends depends on the investment.\newline\newline You’ve heard of this about health-care. The Obama health-care bill seemed to cost \$200 billion—but neither the White House nor the administration said anything about how much it cost. Teich has put its total at \$20 billion.\newline\newline “The big question for me is not just whether Teich decides to spend too much, but whether TARP looks at spending too much,'' says Erin McAllister, executive director of the Teich Public Policy Institute. That calculation helps explain how much Teich spent to fight the opioid epidemic here in Washington D.C., she says, though she says she doesn’t know just how much that money went to a particular cause.\newline\newline Students, for instance, know how to gamble. “They probably knew what they were going to spend,'' she says. “What surprises me is that people don’t know when they buy something that is going to have a bigger impact on an institution than an investment.''\newline\newline Subscribe to Smart Money's Smart Money app—you can download it here.\newline\newline In some cases, Teich says it didn’t know how much of its money was spent, so it didn’t disclose how much the company spent. That likely happened in 2010 and 2012, for “from an investment perspective, you’re basically going the other way to buy something that’s going to have a bigger impact than an investment,'' McAllister says. “You may see something that has less impact than investment, and you may see something that has more impact and less impact than investment.''\newline\newline In New York City—where Teich’s Indiegogo campaign spent nearly \$50 million in fiscal year 2016—it managed to convince advocacy groups to spend the bulk of its money on television ads. The company and its partners estimated that the campaign would finance net \$2 million in television ads over the next 12 months. In other cities, similar things happen year after year, McAllister says.\newline\newline So much depends on the investment. The public-sector infrastructure bought by Teich remains relatively small. Its advertising strategies have changed little over the years, and they continue to attract big-name attention.\newline\newline Of its most recent investments in public health, Ken Farrell, a lawyer with the President’s Office, agrees that Teich has designed previous advertising campaigns that have largely paid off. But McAllister says the payoff could be quite big, depending on how much the company spends.\newline\newline The company’s sticker price has been criticized by elected officials, environmental groups and those who ignore the fossil fuel industry at their peril. But its spending on public health initiatives has been relatively steady—it’s spent almost \$650 million since 2010. Neither Trump or his administration has really pushed for new regulations despite concerns about their potential benefits. And experts aren’t totally convinced that new taxes will improve public health—especially because of the already unacceptably high tolls and highway tolls in the United States. Tax advocates will argue that new taxes will be good for everyone—but are skeptical.\newline\newline “It’s compelling,'' Stan Diem, a political science professor at Wesleyan University says. “But there is no evidence that it is\textbf{\textbar endoftext\textbar}"
\end{samplebox}
\caption{Generated sample from \method~trained on OWT, with a total budget of $T=256$: $64$ unmasking steps, corrector frequency $\omega=1$ and $S=3$ steps per loop.}
\label{fig:owt_256_samples}
\end{figure}

\clearpage

\begin{enumerate}

\item {\bf Claims}
    \item[] Question: Do the main claims made in the abstract and introduction accurately reflect the paper's contributions and scope?
    \item[] Answer: \answerYes{} 
    \item[] Justification: Our abstract and introduction claim that we introduce a self-correction framework for discrete diffusion models.
    This framework and training and sampling algorithms are presented in Sections \ref{sec:methods} and \ref{sec:methods_sampling}.
    We also claim that our method improves the quality-efficiency tradeoff for discrete diffusion models and enables even higher quality generation via inference-time scaling.
    Our experimental results in Section \ref{sec:exp} and Appendix \ref{appsec:exp_results}, strongly support this claim.

\item {\bf Limitations}
    \item[] Question: Does the paper discuss the limitations of the work performed by the authors?
    \item[] Answer: \answerYes{} 
    \item[] Justification: In Section \ref{sec:discussion}, we discuss the main limitation of our work, which entails two model forward passes at training.

\item {\bf Theory assumptions and proofs}
    \item[] Question: For each theoretical result, does the paper provide the full set of assumptions and a complete (and correct) proof?
    \item[] Answer: \answerYes{} 
    \item[] Justification: In Section \ref{subsec:methods_corrector_loss}, we provide the high level outline for the theoretical motivation of our work, and in Appendix \ref{appsec:predictor_corrector}, we provide the full detailed argument and derivation.

    \item {\bf Experimental result reproducibility}
    \item[] Question: Does the paper fully disclose all the information needed to reproduce the main experimental results of the paper to the extent that it affects the main claims and/or conclusions of the paper (regardless of whether the code and data are provided or not)?
    \item[] Answer: \answerYes{} 
    \item[] Justification: High level experimental details are provided in Section \ref{sec:exp}, with full details for reproducibility listed in Appendix \ref{appsec:exp_details}.

\item {\bf Open access to data and code}
    \item[] Question: Does the paper provide open access to the data and code, with sufficient instructions to faithfully reproduce the main experimental results, as described in supplemental material?
    \item[] Answer: \answerNo{} 
    \item[] Justification: Our work does not propose any new datasets or benchmarks. Releasing anonymized code is not feasible at this time.

\item {\bf Experimental setting/details}
    \item[] Question: Does the paper specify all the training and test details (e.g., data splits, hyperparameters, how they were chosen, type of optimizer) necessary to understand the results?
    \item[] Answer: \answerYes{} 
    \item[] Justification: Full hyperparameter details for training and evaluation are described in Appendix \ref{appsec:exp_details}.

\item {\bf Experiment statistical significance}
    \item[] Question: Does the paper report error bars suitably and correctly defined or other appropriate information about the statistical significance of the experiments?
    \item[] Answer: \answerYes{} 
    \item[] Justification: We report error bars for the OpenWebText unconditional generation experiment in Table \ref{tab:owt_error_rows}.
    For the code and math benchmark results, we use greedy decoding and pass@1 for all evaluations, hence error bars are not applicable.
    For the molecule generation, error bars are outside the scope of our computational budget, but the fact that we report results for various guidance parameters and number of decoding steps gives a robust set of results for this experimental setup.

\item {\bf Experiments compute resources}
    \item[] Question: For each experiment, does the paper provide sufficient information on the computer resources (type of compute workers, memory, time of execution) needed to reproduce the experiments?
    \item[] Answer: \answerYes{} 
    \item[] Justification: Compute resources are detailed in Appendix \ref{appsec:exp_details}.
    
\item {\bf Code of ethics}
    \item[] Question: Does the research conducted in the paper conform, in every respect, with the NeurIPS Code of Ethics \url{https://neurips.cc/public/EthicsGuidelines}?
    \item[] Answer: \answerYes{} 
    \item[] Justification: Our work conforms with the NeurIPS code of ethics.

\item {\bf Broader impacts}
    \item[] Question: Does the paper discuss both potential positive societal impacts and negative societal impacts of the work performed?
    \item[] Answer: \answerYes{} 
    \item[] Justification: Broader impacts are discussed in Appendix \ref{appsec:impacts}
    
\item {\bf Safeguards}
    \item[] Question: Does the paper describe safeguards that have been put in place for responsible release of data or models that have a high risk for misuse (e.g., pre-trained language models, image generators, or scraped datasets)?
    \item[] Answer: \answerNA{} 
    \item[] Justification: Although our work is in the area of language modeling, which when scaled is subject to potential misuse, the scope and model sizes that we explore are too small to pose a meaningful risk of misuse.

\item {\bf Licenses for existing assets}
    \item[] Question: Are the creators or original owners of assets (e.g., code, data, models), used in the paper, properly credited and are the license and terms of use explicitly mentioned and properly respected?
    \item[] Answer: \answerYes{} 
    \item[] Justification: All libraries and datasets, and their corresponding licenses, are detailed in Appendix \ref{appsec:assets}.

\item {\bf New assets}
    \item[] Question: Are new assets introduced in the paper well documented and is the documentation provided alongside the assets?
    \item[] Answer: \answerNA{} 
    \item[] Justification: We do not release new assets in this manuscript.

\item {\bf Crowdsourcing and research with human subjects}
    \item[] Question: For crowdsourcing experiments and research with human subjects, does the paper include the full text of instructions given to participants and screenshots, if applicable, as well as details about compensation (if any)? 
    \item[] Answer: \answerNA{} 
    \item[] Justification: We do not conduct research with crowdsourcing or human subjects.

\item {\bf Institutional review board (IRB) approvals or equivalent for research with human subjects}
    \item[] Question: Does the paper describe potential risks incurred by study participants, whether such risks were disclosed to the subjects, and whether Institutional Review Board (IRB) approvals (or an equivalent approval/review based on the requirements of your country or institution) were obtained?
    \item[] Answer: \answerNA{} 
    \item[] Justification: We do not conduct research with crowdsourcing or human subjects.

\item {\bf Declaration of LLM usage}
    \item[] Question: Does the paper describe the usage of LLMs if it is an important, original, or non-standard component of the core methods in this research? Note that if the LLM is used only for writing, editing, or formatting purposes and does \emph{not} impact the core methodology, scientific rigor, or originality of the research, declaration is not required.
    \item[] Answer: \answerYes{} 
    \item[] Justification: LLM usage is detailed in Appendix \ref{appsec:llm}.

\end{enumerate}

\end{document}